\documentclass{article}

\usepackage{microtype}
\usepackage{graphicx}
\usepackage{subfigure}
\usepackage{booktabs} 

\usepackage{hyperref}

\usepackage[accepted]{icml2020}

\icmltitlerunning{Combining Pessimism with Optimism for Robust and Efficient Model-Based Deep Reinforcement Learning }

\usepackage{amssymb}
\usepackage{xspace}        
\usepackage{mathtools}
\usepackage{color}
\usepackage[capitalise]{cleveref}
\usepackage{url}
\usepackage{multicol}

\usepackage{amsmath,amsfonts,bm}

\def\1{\bm{1}}

\DeclareMathAlphabet{\mathsfit}{\encodingdefault}{\sfdefault}{m}{sl}
\SetMathAlphabet{\mathsfit}{bold}{\encodingdefault}{\sfdefault}{bx}{n}

\DeclareMathOperator*{\argmax}{arg\,max}
\DeclareMathOperator*{\argmin}{arg\,min}

\newcommand{\iid}{\textit{i.i.d.}\xspace}

\newcommand{\E}[2][]{\mathbb{E}_{#1}\mathopen{}\left[#2\right]\mathclose{}}

\newcommand{\Or}[1]{\mathcal{O}\mathopen{}\left(#1\right)\mathclose{}}

\newcommand{\R}{\mathbb{R}}

\newcommand{\mb}[1]{\mathbf{#1}}        

\renewcommand{\mid}{\,|\,}

\newcommand{\alg}{\texttt{RH-UCRL}\xspace}
\newcommand{\hucrl}{\texttt{H-UCRL}\xspace}
\newcommand{\rhucrl}{\alg}

\newcommand{\x}{\mb{s}}
\renewcommand{\u}{\mb{a}}
\newcommand{\X}{\mathcal{S}}
\newcommand{\U}{\mathcal{A}}

\newcommand{\hi}{h}                  
\newcommand{\Hi}{H}
\newcommand{\ti}{t}                
\newcommand{\Ti}{T}

\newcommand{\noise}{\bm{\omega}}

\newcommand{\etao}{{\eta^{(o)}}}

\newcommand{\fo}{{f^{(o)}}}

\newcommand{\etap}{{\eta^{(p)}}}

\newcommand{\fp}{{f^{(p)}}}

\newcommand{\uadv}{\bar{\u}}
\newcommand{\Uadv}{\bar{\U}} 
\newcommand{\piadv}{\bar{\pi}} 
\newcommand{\Piadv}{\bar{\Pi}} 

\newcommand{\pifinal}{\hat{\pi}_T}

\newcommand{\nstate}{\ensuremath{p}}
\newcommand{\ninp}{\ensuremath{q}}
\newcommand{\ninpadv}{{\bar{\ninp}}}

\newcommand{\dataset}{\ensuremath{\mathcal{D}}}
\newcommand{\modelposterior}{\ensuremath{p(\tilde{f} \mid \dataset_{1:\ti})}}

\newcommand{\bsigma}{\bm{\sigma}}
\newcommand{\bSigma}{\bm{\Sigma}}
\newcommand{\bmu}{\bm{\mu}}

\usepackage{amsthm}
\usepackage{thmtools}
\usepackage{thm-restate}
\theoremstyle{plain}

\newtheorem{lemma}{Lemma}

\theoremstyle{definition}
\newtheorem{assumption}{Assumption}
\newtheorem{definition}{Definition}

\theoremstyle{remark}

\crefname{assumption}{assumption}{assumptions}
\Crefname{assumption}{Assumption}{Assumptions}

\begin{document}

\twocolumn[
\icmltitle{Combining Pessimism with Optimism for Robust and Efficient \\Model-Based Deep Reinforcement Learning}




\icmlsetsymbol{equal}{*}

\begin{icmlauthorlist}
\icmlauthor{Sebastian Curi}{to}
\icmlauthor{Ilija Bogunovic}{to}
\icmlauthor{Andreas Krause}{to}
\end{icmlauthorlist}

\icmlaffiliation{to}{Department of Computer Science, ETH, Z\"urich, Switzerland}

\icmlcorrespondingauthor{Sebastian Curi}{sebastian.curi@inf.ethz.ch}


\vskip 0.3in]



\printAffiliationsAndNotice{}  

\begin{abstract}
    In real-world tasks, reinforcement learning (RL) agents frequently encounter situations that are not present during training time.
    To ensure reliable performance, the RL agents need to exhibit {\em robustness} against worst-case situations. 
    The robust RL framework addresses this challenge via a worst-case optimization between an agent and an adversary.
    Previous robust RL algorithms are either sample inefficient, lack robustness guarantees, or do not scale to large problems.
    We propose the \emph{Robust Hallucinated Upper-Confidence} RL (\alg) algorithm to \em {provably} solve this problem while attaining \em {near-optimal} sample complexity guarantees. \alg is a model-based reinforcement learning (MBRL) algorithm that effectively distinguishes between epistemic and aleatoric uncertainty, and {\em efficiently} explores both the agent and adversary decision spaces during policy learning. 
    We scale \alg to complex tasks via neural networks ensemble models as well as neural network policies. 
    Experimentally, we demonstrate that \alg outperforms other robust deep RL algorithms in a variety of adversarial environments. \looseness=-1
\end{abstract}

\section{Introduction} \label{sec:introduction}

\looseness -1     A central challenge when deploying Reinforcement Learning (RL) agents in real environments is their robustness \citep{dulac2019challenges}.
As a motivating example, consider designing a braking system on an autonomous car. 
As this is a highly complex task, we want to \emph{learn} a policy that performs this maneuver.
Even if various real-world conditions can be simulated during the training time, it is infeasible to consider \emph{all} possible ones such as road conditions, brightness, tire pressure, laden weight, or actuator wear, as these can all vary over time in potentially unpredictable ways. 
The main goal is then to learn a policy that \emph{provably} brakes in a \emph{robust} fashion so that, even if faced with new conditions, it performs reliably. 
Borrowing from the robust control perspective (e.g, as considered in $\mathcal{H}_\infty$ control \citet{bacsar2008h}), we model robustness with an adversary that is allowed to perform the worst-case disturbance for the given policy. 


    While there are theoretical approaches for robust RL that offer sample complexity guarantees (see \cref{sec:related_work}), the existing algorithms are highly impractical.
    On the other hand, there are empirically motivated heuristic approaches that lack provable robustness. 
    We bridge this gap by proposing an algorithm that simultaneously enjoys rigorous theoretical guarantees, yet can be applied to complex tasks. \looseness=-1

    We develop the \emph{Robust Hallucinated Upper-Confidence Reinforcement Learning} (\rhucrl) algorithm for obtaining robust RL policies.
    It relies on a probabilistic model that can distinguish between epistemic uncertainty (that arises from the lack of data) and aleatoric uncertainty (stochastic noise) \citep{der2009aleatory}. 
    A key algorithmic principle behind \rhucrl is \emph{hallucination}: In particular, the {\em agent} hallucinates an additional control input to {\em maximize} an {\em optimistic} estimate of the robust performance whereas the {\em adversary} hallucinates an additional control input to {\em minimize} a {\em pessimistic} estimate of the robust performance.
    The amount of ``hallucination'' is limited by the epistemic uncertainty of the model and it decreases as the learning algorithm collects more data. 
    In a number of experiments, we show that our algorithm exhibits robust performance and  outperforms previous approaches on various popular RL benchmarks. In \Cref{fig:inverted_final}, we demonstrate how \alg outperforms baselines in a pendulum swing-up control task. In particular, our main observations are: (i) Robust baselines specifically designed for different settings do not explore sufficiently and fail to find a swing-up maneuver; (ii) Non-robust algorithm finds the swing-up strategy, but as adversarial power increases its performance significantly deteriorates, whereas \alg exhibits higher robustness to {\em worst-case} perturbations.\looseness=-1
    
    \begin{figure*}[t]
  \centering\includegraphics[width=\textwidth]{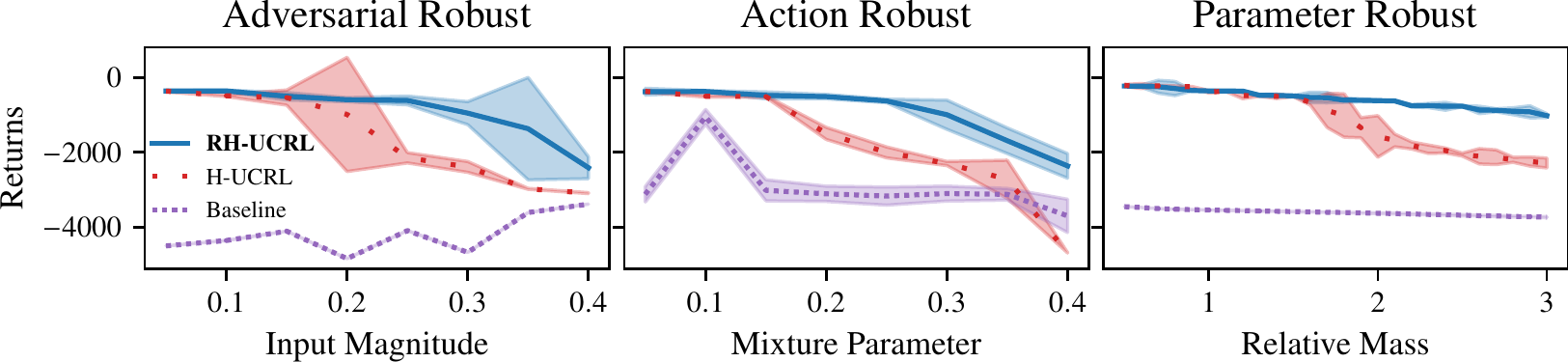}
    \vspace{-1em}
    
    \caption{
    Performance of \rhucrl (this work), \hucrl \cite{curi2020efficient}, and a baseline with no exploration in adversarial-, action-, and parameter-robust settings on a pendulum swing-up task. 
    The baseline never learns a successful swing-up strategy due to insufficient exploration.
    \rhucrl and \hucrl learn a swing-up strategy but \rhucrl outperforms \hucrl as the perturbation increases.\looseness=-1
  }
  \label{fig:inverted_final}
\end{figure*}

    \textbf{Main contributions.} 
    \looseness -1 We design \rhucrl, the first practical \emph{provably} robust RL algorithm that is: (i) \emph{sample-efficient}, (ii) compatible with \emph{deep models}, and (iii) simulator-free as it addresses exploration on a real system. We establish rigorous general sample-complexity and regret guarantees for our algorithm, and we specialize them to Gaussian Process models, hence obtaining sublinear robust regret guarantees. While previous robust RL works have focused on different individual settings, we gather and summarize them all for the first time, and show particular instantiations of our algorithm in each of them: (i) Adversarial-robust RL, (ii) Action-robust RL, and (iii) Parameter-robust RL. Finally, we provide experiments that include different environments and settings, and we empirically demonstrate that \rhucrl outperforms or successfully competes with the state-of-the-art deep robust RL algorithms and other baselines.

    
\section{Related Work}
\label{sec:related_work}
\textbf{Robust RL.}
Robust Reinforcement Learning \citep{iyengar2005robust,nilim2005robust,wiesemann2013robust}
 typically uses the Zero-Sum Markov Games formalism introduced by \citet{littman1994markov,littman1996generalized}. 
From a control engineering perspective, this is known as $\mathcal{H}_\infty$ control \citep{bacsar2008h}, and \citet{bemporad2003min} solve this problem for {\em known} linear systems where the optimal robust policy is an affine function of the state. 
For Markov Games with unknown systems, \citet{lagoudakis2002value} introduce approximate value iteration whereas \citet{tamar2014scaling} and \citet{perolat2015approximate} present an approximate policy iteration scheme.
Although these works present error propagation schemes restricted to the linear setting, they do not address the problem of exploration explicitly, and instead assume access to a sampling distribution that has sufficient state coverage. 
\citet{zhang2020model} propose a model-based algorithm for finding robust policies assuming access to a simulator that is able to sample at arbitrary state-action pairs. 
Finally, \citet{bai2020provable} recently introduce an algorithm that provably and efficiently outputs a policy, but it is limited to the tabular setting. 
Instead, our approach does not require a generative model and considers the full exploration problem in a finite horizon scenario. 
Furthermore, our algorithm does not require tabular nor linear function approximation and it is compatible with deep neural networks dynamical models.

\textbf{Minimax Optimization Algorithms.}
\citet{pinto2017robust} propose to solve the minimax optimization via a stochastic gradient descent approach for both players for {\em adversarial-robust} RL algorithms. 
\citet{tessler2019action} introduce {\em action-robust} RL and policy and value iteration algorithms to solve these problems. 
\citet{rajeswaran2016epopt} introduce the \texttt{EPOpt} to solve {\em parameter-robust} problems. 
Finally, \citet{kamalaruban2020robust} propose to use a Stochastic-Gradient Langevin Dynamics algorithm to solve such problems via sampling instead of optimization. 
These algorithms are generally sample-inefficient as they do not explicitly explore but, given a model, they could be used to optimize a policy.
\looseness=-1

\textbf{Provable Model-Based RL.}
Model-Based RL has better empirical sample-efficiency than model-free variants \citep{deisenroth2011pilco,chua2018deep}. 
Furthermore, the celebrated UCRL algorithm by \citet{auer2009near} has provable sample-efficiency guarantees, albeit being intractable for non-tabular environments. 
Recently, \citet{curi2020efficient} instantiate the UCRL algorithm in continuous control problems using hallucinated control to explore.
We build upon their \hucrl algorithm and strictly generalize it to the robust RL setting.
Our robust policy search strategy is inspired by the work in the related bandit setting with Gaussian Processes, where \texttt{Stable-OPT} \citep{bogunovic2018adversarially} combines pessimism with optimism to similarly generalize the non-robust \texttt{GP-UCB} \citep{srinivas2010gaussian} algorithm and provably discover robust designs.


\section{Problem Statement}
    We consider a stochastic environment with states $\x \in \X \subseteq \R^\nstate$, agent actions $\u \in \U \subset \R^\ninp$, adversary actions $\uadv \in \Uadv \subset \R^\ninpadv$, and \iid~additive transition noise vector $\noise_\hi \in \R^\nstate$. 
    Both action sets are assumed to be compact, and the dynamics are given by:
    \begin{equation}
    	\x_{\hi+1} = f(\x_\hi, \u_\hi, \uadv_\hi) + \noise_\hi
    	\label{eq:stochastic_dynamic_system_additive}
    \end{equation}
    with $f \colon \X \times \U \times \Uadv \to \X$. 
    We assume the true dynamics $f$ are \emph{unknown} and consider the episodic setting over a finite time horizon $\Hi$. 
    After every episode (i.e., every $\Hi$ time steps), the system is reset to a known state $\x_0$. 
    In this work, we make the following assumptions regarding the stochastic environment and unknown dynamics:
    \begin{assumption}
        The dynamics $f$ in \cref{eq:stochastic_dynamic_system_additive} are $L_f$-Lipschitz continuous and, for all $\hi \in \lbrace 0, \dots,\Hi-1 \rbrace$, the individual entries of the noise vector $\noise_\hi$ are \iid~$\sigma$-sub-Gaussian.
      \label{as:dynamics_f_lipschitz}
    \end{assumption}
    At every time-step, the system returns a deterministic reward $r(\x_\hi, \u_\hi, \uadv_\hi)$, where $r: \X \times \U \times \Uadv \to \R$ is assumed to be known to the agent. 
    We consider time-homogeneous 
    agent policies $\pi \in \Pi$, $\pi \colon \X \to \U$, that select actions according to $\u_\hi = \pi(\x_\hi)$. Similarly, we consider adversary policies $\pi \in \Piadv$ on the common state space, i.e., $\piadv \colon \X \to \Uadv$, that select actions as $\uadv_\hi = \piadv(\x_\hi)$.
    For the sake of simplicity, we omit straightforward extensions such as initial-state distributions, unknown reward functions, or time-indexed policies that can be adressed using standard techniques \citep{chowdhury2019online}. 
    For now, we leave both $\Pi$ and $\Piadv$ unspecified, but in \Cref{sec:experiments}, we parameterize them via neural networks. \looseness=-1
    
    
    

    The performance of a pair of policies $(\pi, \piadv)$ on a given dynamical system $\tilde{f}$ is the episodic expected sum of returns:
    \begin{align}
        J(\tilde{f}, \pi, \piadv) &\coloneqq \E[\tau_{\tilde{f}, \pi, \piadv}]{\sum_{\hi=0}^{\Hi} r(\x_\hi, \u_\hi, \uadv_\hi) \, \bigg| \, \x_0}, \\
        \text{s.t. }\;  \x_{\hi+1}& = \tilde{f}(\x_\hi, \u_\hi, \uadv_\hi) + \noise_\hi, \nonumber
        \label{eq:performance} 
    \end{align}
    where $\tau_{\tilde{f},\pi,\piadv} = \left\{ (\x_{\hi-1}, \u_{\hi-1}, \uadv_{\hi-1}),  \x_{\hi} \right\}_{\hi=0}^\Hi$ is a random trajectory induced by the stochastic noise $\noise$, the dynamics $\tilde{f}$, and the policies $\pi$ and $\piadv$. 

    We use $\pi^{\star}$ to denote the optimal deterministic \emph{robust} policy from set $\Pi$ in case of true dynamics $f$, i.e.,
    \begin{equation}
        \pi^\star \in \arg \max_{\pi \in \Pi} \min_{\piadv \in \Piadv} J(f, \pi, \piadv). \label{eq:objective}
    \end{equation}
    Even when the true system dynamics are known, finding a robust policy is generally a challenging task for arbitrary policy sets, reward and transition functions. In the rest, we make an assumption that \cref{eq:objective} can be solved for a given dynamics, and in \Cref{ssec:practical_implementation}, we propose a concrete problem instantiation and algorithmic solution.  

    \textbf{Learning protocol.} 
        We consider the episodic setting in which, at every episode $\ti$, the learning algorithm selects both the agent's $\pi_\ti$ and a fictitious adversary's $\piadv_\ti$ policies. 
        The pair of policies ($\pi_\ti$, $\piadv_\ti$) is then deployed on the true system $f$, and a single realization of the trajectory $\tau_{f, \pi_\ti, \piadv_\ti}$ is observed and used to update the underlying statistical model. We summarize the general learning protocol in Algorithm~\ref{alg:R-HUCRL_protocol}.
        In the braking system example, this learning protocol implies that during training we are allowed to execute braking maneuvers as well as possible adversarial policies, e.g., changing the braking surface. 
        The execution of both policies during training is crucial to guarantee robust performance: The learner can actively look for the \emph{worst-case} adversarial policies that it might encounter during deployment and learn what to do when faced upon them.

        \begin{algorithm}[t!]
        \caption{Robust Model-based Reinforcement Learning}
        \label{alg:R-HUCRL_protocol}
            \begin{algorithmic}[1]
                \STATE \textbf{Require:} Calibrated dynamical model, reward function $r(\x_\hi, \u_\hi, \uadv_\hi)$, horizon $\Hi$, initial state $\x_0$
                \FOR{$t=1,2, \dots$}
                    \STATE Select ($\pi_\ti$,$\piadv_\ti$) using the current dynamical model
                    \FOR{$\hi = 1, \dots, \Hi - 1$}
                        \STATE $\x_{\hi,\ti} = f(\x_{\hi-1,\ti}, \pi_\ti(\x_{\hi-1,\ti}),  \piadv_\ti(\x_{\hi-1,\ti})) + \noise_{\hi,\ti}$ 
                    \ENDFOR 
                    \STATE Update statistical dynamical model with the $\Hi$ observations $\left\{ (\x_{\hi-1,t}, \u_{\hi-1,t}, \uadv_{\hi-1,t}),  \x_{\hi,t} \right\}_{\hi=0}^\Hi$
                    \STATE Reset the system to $\x_{0,t+1} = \x_0$
                \ENDFOR
            \end{algorithmic}
        \end{algorithm}
    
    \textbf{Performance metric.} 
        For a small fixed $\epsilon>0$, the goal is to output a robust policy $\pi_{\Ti}$ after $\Ti$ episodes such that:
        \begin{equation}
            \min_{\piadv \in \Piadv} J(f, \pi_\Ti, \piadv) \geq \min_{\piadv \in \Piadv} J(f, \pi^{\star}, \piadv) - \epsilon, \label{eq:PAC}
        \end{equation}
        where $\pi^{\star}$ is defined as in~\cref{eq:objective}. Hence, we consider the task of near-optimal robust policy identification, but we note that one can also measure the performance in terms of the robust cumulative regret as discussed in \cref{sec:theoretical_analysis}. Thus, the goal is to output the agent's policy with near-optimal robust performance when facing its own \emph{worst-case} adversary, and the adversary selects $\piadv$ \emph{after} the agent selects $\pi_\Ti$. Note that this is a strong   er robustness notion than just considering the worst-case adversary of the optimal policy, since, by letting $\piadv^* \in  \argmin_{\piadv \in \Piadv} J(f, \pi^{\star}, \piadv)$, we have $J(f, \pi_\Ti, \piadv^*) \geq \min_{\piadv \in \Piadv} J(f, \pi_\Ti, \piadv)$.
        \looseness=-1
        
 \textbf{Statistical Model.}
        In this work, we take a {\em model-based reinforcement learning (MBRL)} approach. That is, to learn the dynamics and discover a near-optimal robust policy, we consider algorithms that model and sequentially learn about $f$ from noisy state observations. The agent makes use of the observed data (collected within an episode) to simultaneously improve its estimate of the true dynamics $f$. 
        
        We use statistical estimation to probabilistically reason about dynamical models $\tilde{f}$ that are compatible with the observed data $\dataset_{1:\ti} = \left\{ \tau_{f, \pi_{\ti'}, \piadv_{\ti'}} \right\}_{\ti'=1}^{\ti}$. This can be done, e.g., by frequentist estimation of mean $\bmu_\ti(\x, \u, \uadv)$ and confidence $\bSigma^2_\ti(\x, \u, \uadv)$ estimators, or by taking a Bayesian perspective and considering the posterior distribution $\modelposterior$ over dynamical models that leads to $\bmu_{\ti}(\x, \u, \uadv) = \mathbb{E}_{\tilde{f} \sim \modelposterior} [\tilde{f}(\x, \u, \uadv) ]$ and $\bSigma_\ti^2(\x, \u, \uadv)= \mathrm{Var}[ \tilde{f}(\x, \u, \uadv)]$. In any case, we require the model to be {\em calibrated}: \looseness=-1

        \begin{assumption}[Calibrated model]
            The statistical model is {\em calibrated} w.r.t.~$f$ in \cref{eq:stochastic_dynamic_system_additive}, so that with $\bsigma_\ti(\cdot) = \mathrm{diag}(\bSigma_\ti(\cdot))$ and a non-decreasing sequence of parameters $\lbrace \beta_\ti \rbrace_{t\geq 1} \in \R_{>0}$, each depending on $\delta \in (0,1)$, it holds jointly for all $\ti \geq 1$ and $\x, \u, \uadv \in \X \times \U \times \Uadv$ that $| f(\x, \u, \uadv) - \bmu_{\ti-1}(\x, \u, \uadv) | \leq \beta_{\ti} \bsigma_{\ti-1}(\x, \u, \uadv)$ element-wise, with probability at least $1 - \delta$.\label{as:well_calibrated_model}
        \end{assumption}
        
        This assumption is important for exploration: if the model is not calibrated, then using the epistemic uncertainty of such a model will not \emph{provably} guide exploration.
        For dynamics with finite norm in a known RKHS space, \Cref{as:well_calibrated_model} is satisfied \citep{srinivas2010gaussian,chowdhury2017kernelized}. 
        In case of neural network models, we can recalibrate one-step ahead predictions \citep{malik2019calibrated}. 

\section{The Robust \hucrl Algorithm (\rhucrl) } \label{sec:RHUCRL}

We now develop our \rhucrl algorithm that can be used in Algorithm~\ref{alg:R-HUCRL_protocol} (at Line 3), for selecting policies $\pi_\ti$ and $\piadv_\ti$. \rhucrl takes the sequence of confidence parameters $\lbrace \beta_t \rbrace_{t\geq 1}$ from \Cref{as:well_calibrated_model} as input. The main idea is to use our probabilistic model of $f$ to {\em optimistically} select $\pi_\ti$ and {\em pessimistically} select $\piadv_t$ w.r.t.~all plausible dynamics.\looseness=-1

\subsection{Optimistic and Pessimistic Policy Evaluation}
    For any two policies $\pi$ and $\piadv$, we provide the \textbf{(o)}ptimistic and \textbf{(p)}essimistic estimate of $J(f, \pi, \piadv)$ at time $\ti$, and we denote them with $J_\ti^{(o)}(\pi, \piadv)$ and $J_\ti^{(p)}(\pi, \piadv)$, respectively. 
    Hereby, the optimistic estimate is given by:
    \begin{subequations}
    \begin{align}
       J_\ti^{(o)} (\pi, \piadv) &\coloneqq \max_{\etao} J (f^{(o)}, \pi, \piadv), \label{eq:optimistic_performance} \\
        \text{s.t. }\; & \fo(\x, \u, \uadv) = \mu_{\ti-1} \big(\x, \u, \uadv) \nonumber  \\ &\quad + \beta_{\ti} \etao(\x, \u, \uadv) \Sigma_{\ti-1}^{1/2}(\x, \u, \uadv). \label{eq:optimistic_transition}
    \end{align} \label{eq:optimistic}%
    \end{subequations}
    Similarly, the pessimistic estimate is given by:
    \begin{subequations}
    \begin{align}
        J_\ti^{(p)}(\pi, \piadv) &\ \coloneqq  \min_{\etap}  J(f^{(p)}, \pi, \piadv) \label{eq:pessimistic_performance} \\ \nonumber
        \text{s.t. }\; & \fp(\x, \u, \uadv) = \mu_{\ti-1} \big(\x, \u, \uadv) \\
        &\quad + \beta_{\ti} \etap(\x, \u, \uadv) \Sigma_{\ti-1}^{1/2}(\x, \u, \uadv).\label{eq:pessimistic_transition}
    \end{align} \label{eq:pessimistic}%
    \end{subequations}
    We note that $J_\ti^{(o)}$ and $J_\ti^{(p)}$ represent upper and lower bounds on the performance of the policies $\pi$, $\piadv$ in case of the true dynamics $f$. 
    These estimates are computed by finding the most optimistic (pessimistic) dynamics compatible with the data. Note that the optimistic/pessimistic outcome is selected via decision variables $\etao/\etap: \X \times \U \times \Uadv \to [-1, 1]^\nstate$, which are functions of the state as well as actions of both players. These select among all plausible outcomes of the dynamics bounded within the epistemic uncertainty over $f$.
    When the policies are fixed and clear from context, with slight abuse of notation we write $\eta(\x, \u, \uadv) = \eta(\x, \pi(\x), \piadv(\x)) = \eta(\x)$. 
    A crucial observation is that both \cref{eq:optimistic,eq:pessimistic} can be viewed as two optimal control problems, where the decision variables $\etao/\etap$ are hallucinated control policies, whose effect is bounded by the model epistemic uncertainty. 
    We can use optimal control algorithms \citep{camacho2013model} to maximize/minimize the sum of rewards following the reparameterized dynamics $\fo/\fp$. \looseness=-1

\subsection{The \texorpdfstring{\rhucrl} Algorithm}
    Given both the pessimistic and optimistic performance estimates from the previous section, we are now ready to state our algorithm. At each episode $\ti$, \rhucrl selects the agent and adversary policies as follows:
    \begin{subequations}
    \begin{align}
        \pi_\ti &\in \argmax_{\pi \in \Pi} \min_{\piadv \in \Piadv} J_{\ti}^{(o)}(\pi, \piadv), \label{eq:rhucrl:learner} \\ 
        \piadv_\ti &\in \argmin_{\piadv \in \Piadv} J_{\ti}^{(p)}(\pi_\ti, \piadv) .  \label{eq:rhucrl:adversary}
    \end{align} \label{eq:rhucrl}%
    \end{subequations}
    \looseness -1 
    Thus, \rhucrl selects the most optimistic robust policy for the agent player in \cref{eq:rhucrl:learner}. 
    The adversary player picks the most pessimistic policy given the selected agent policy in \cref{eq:rhucrl:adversary}. 
    When the adversarial policy space $\Piadv$ is a singleton, \rhucrl reduces to the \hucrl algorithm.
    
    Finally, after a total of $\Ti$ episodes, the algorithm outputs an agent policy $\pifinal$ given by: 
    \begin{equation}
        \pifinal = \pi_{\ti^\star}\; \text{ s.t. } \; \ti^\star \in \argmax_{\ti \in \{1, \ldots, \Ti\} } J_{\ti}^{(p)}(\pi_\ti, \piadv_\ti). \label{eq:rhucrl:output}
    \end{equation}
    There is no extra computational cost in identifying the output policy as $J_{\ti}^{(p)}(\pi_\ti, \piadv_\ti)$ is already computed by the learner in \cref{eq:rhucrl:adversary} in every episode $\ti$. Thus, the algorithm simply returns the encountered agent policy with maximum pessimistic robust performance.

\subsection{Practical Implementation} \label{ssec:practical_implementation}
    \looseness -1 To implement \alg, we parameterize $\pi$, $\piadv$, and $\eta$ using neural network policies and train them using actor-critic algorithms.
    We remark that $\piadv$ in the agent optimization \eqref{eq:rhucrl:learner} and $\piadv$ in the adversary optimization \eqref{eq:rhucrl:adversary} are different, as are the pessimistic and optimistic hallucination policies $\etap$ and $\etao$. 
    In particular, we learn a critic via fitted Q-iteration \citep{perolat2015approximate,antos2008fitted} and then differentiate through the critic using pathwise gradients \citep{mohamed2019monte,silver2014deterministic} using stochastic gradient ascent for the agent and stochastic gradient descent for the adversary. 

\subsection{Theoretical Analysis}
    \label{sec:theoretical_analysis}
    In this section, we theoretically analyze the performance of the \rhucrl algorithm. First, we use the notion of \emph{robust cumulative regret}\footnote{Similar notions of robust cumulative regret have been analyzed before in bandit optimization  \citep[see, e.g.,][]{kirschner2020distributionally}.} $$R_{\Ti} = \sum_{t=1}^T \min_{\piadv \in \Piadv} J(f, \pi^\star, \piadv) - \min_{\piadv \in \Piadv} J(f, \pi_\ti, \piadv),$$  
    \looseness -1 which measures the difference in performance between the optimal robust policy and the sequence of agent's policies $\lbrace \pi_1, \dots, \pi_\Ti \rbrace$ selected at every episode in \cref{eq:rhucrl:learner}. Below (see \Cref{thm:exploration:regret:general_regret_bound}) we establish that \rhucrl achieves sublinear regret, i.e., $R_{\Ti}/\Ti\to 0$ for $T\to \infty$.
    In addition to the robust regret notion, we also analyze the recommendation rule of \rhucrl via \cref{eq:rhucrl:output}, and the number of episodes $T$ required to output a near-optimal robust policy (see \Cref{cor:simple_regret:general_regret_bound}). We start by analyzing a general robust model-based RL framework, and later on, we demonstrate the utility of the obtained results by specializing them to the important special case of Gaussian Process dynamics models. We defer all the proofs from this section to \Cref{sec:proofs_main}. Before stating our main theoretical results, we introduce some additional assumptions:
    \begin{assumption}[Lipschitz continuity]
        \looseness -1 
        At every episode $\ti$, the functions $\bsigma_\ti$, any agent's and adversary's policies $\pi_{\ti} \in \Pi$, $\piadv_{\ti} \in \Piadv$, and the reward $r(\cdot, \cdot, \cdot)$ are Lipschitz continuous with respective constants $L_\sigma$, $L_\pi$, $L_{\piadv}$ and $L_r$. 
        \label{as:model_predictions_lipschitz}
        \label{as:pi_lipschitz}
        \label{as:piadv_lipschit}
        \label{as:action_set_compact}
        \label{as:reward_lipschitz}
    \end{assumption}
    The previous assumption is mild and has been used in non-robust model-based RL,  see, e.g., \citet{curi2020efficient}, where it is noted that neural networks with Lipschitz-continuous non-linearities (or GPs with Lipschitz continuous kernels) output Lipschitz-continuous predictions.
    Furthermore, the policy classes $\Pi$ and $\Piadv$, as well as the reward functions, are typically known and designed in a way that is compatible with the previous assumption. 
    
    Both the robust regret and sample complexity rates that we analyze depend on the difficulty of learning the underlying statistical model. Models that are easy to learn typically require fewer samples and allow algorithms to make better decisions sooner. To express the difficulty of learning the imposed calibrated model class, we use the following model-based complexity measure:
    \begin{equation}
        \Gamma_\Ti \coloneqq \max_{\tilde{\dataset}_{1:\Ti}} \sum_{\ti=1}^{\Ti} \sum_{(\x,\u, \uadv) \in \tilde{\dataset}_{t} } \|\bsigma_{\ti-1}(\x, \u, \uadv)\|^2_2 \label{eq:complexity_measure}
    \end{equation}
    where each $\tilde{\dataset}_{\ti} \subset \lbrace \X \times \U \times \Uadv \rbrace^H$. This quantity has a worst-case flavor as it considers the data (collected during $\Ti$ episodes by any algorithm) that lead to maximal total predictive uncertainty of the model. For the special case of RKHS/GP dynamics models, we show below that this quantity can be effectively bounded, and the bound is sublinear (in the number of episodes $\Ti$) for most commonly used kernel functions. \looseness=-1
    
    

\begin{figure*}[t]
  \centering\includegraphics[width=\textwidth]{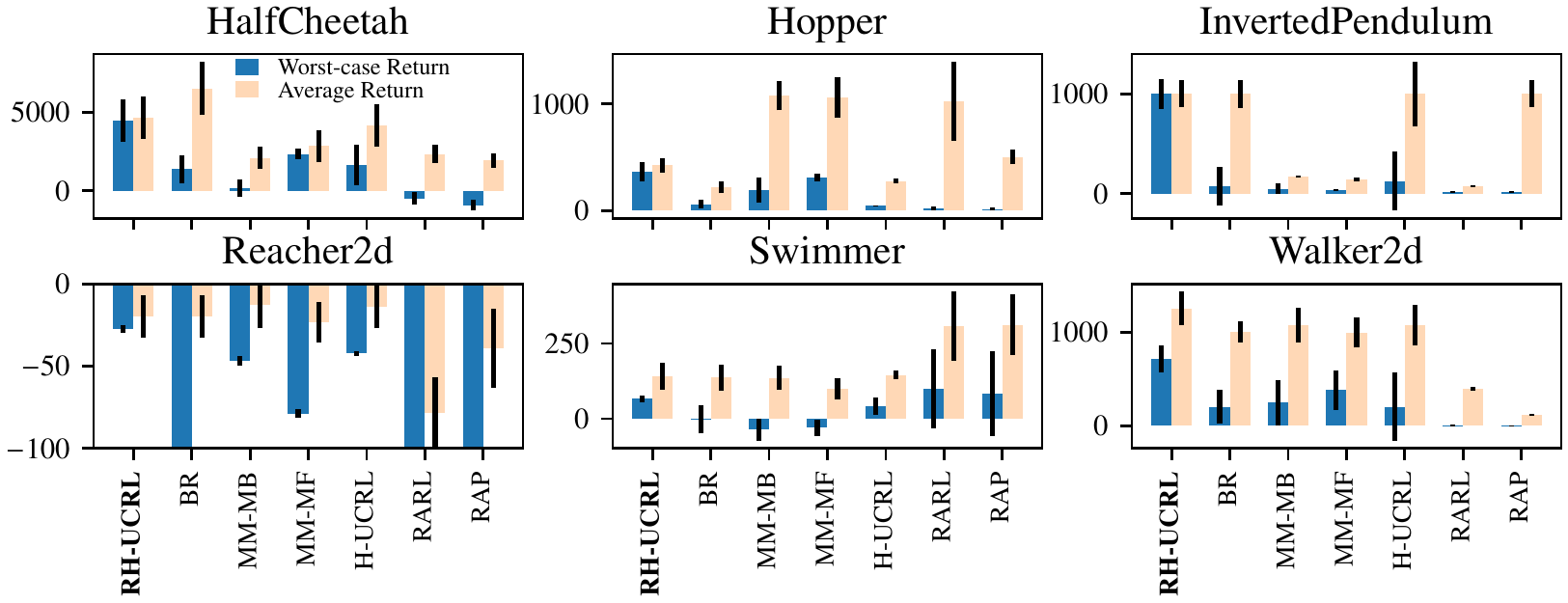}
    \vspace{-1em}
    
    \caption{
    Worst-case and average return of different algorithms in the Adversarial-Robust Setting in Mujoco tasks. 
    \alg outperforms the other algorithms in terms of worst-case return. 
    The non-robust baseline, \hucrl, has good average performance but poor worst-case performance (e.g., Inverted Pendulum). 
    The deep robust RL baselines have worse sample complexity and often underperform. 
    Our ablations are also non-robust, since exploration of both {\em agent} and {\em adversary} is crucial here to achieve robust performance.
  }
  \label{fig:adversarial_robust_final}
\end{figure*}

    \textbf{General results.}
    Now, we can state the main result of this section. In the following theorem, we bound the robust cumulative regret incurred by the policies from \cref{eq:rhucrl:learner}.
    \begin{restatable}{theorem}{generalregretbound} \label{thm:exploration:regret:general_regret_bound}
      Under \Cref{as:dynamics_f_lipschitz,as:pi_lipschitz,as:reward_lipschitz,as:well_calibrated_model,as:model_predictions_lipschitz}, let $C= (1 + L_f + 2 L_{\sigma})(1 + L_{\pi}^2 + L_{\piadv}^2)^{1/2}$ and  
      let $\x_{\ti,\hi} \in \X$, $\u_{\ti,\hi} \in \U$, $\uadv_{\ti,\hi} \in \Uadv$ for all $\ti,\hi>0$. Then, for any fixed $\Hi \geq 1$, with probability at least $1-\delta$, the robust cumulative regret of \alg is upper bounded by:
      $$
        R_\Ti = \Or{L_r C^\Hi \beta_{\Ti}^\Hi \Hi^{3/2} \sqrt{ \Ti   \, \Gamma_\Ti}}.
      $$
    \end{restatable}
    This regret bound shows that \rhucrl achieves sublinear robust regret when $\beta_{\Ti}^\Hi \sqrt{\Gamma_\Ti} = o(\sqrt{T})$. Below, we show a concrete example of GP models where this is indeed the case. The obtained bound also depends on the Lipschitz constants from \Cref{as:piadv_lipschit}, as well as the episode length $H$ that we assume is constant.
    The dependency of the regret bound on the problem dimension is hidden in $\Gamma_\Ti$, while $\beta_T$ depends also on $\delta $ (see \Cref{as:well_calibrated_model}). \looseness=-1
    
    Next, we characterize the number of episodes (samples) required by \rhucrl to output $\epsilon$-optimal robust policy. Our analysis upper bounds the optimal robust performance according to the confidence bounds from \Cref{as:well_calibrated_model}, but also addresses the challenge of characterizing the impact of exploring different adversary policies in \cref{eq:rhucrl:adversary}.
    
    \begin{restatable}{corollary}{simpleregretbound}
    \label{cor:simple_regret:general_regret_bound}
    Consider the assumptions and setup of \Cref{thm:exploration:regret:general_regret_bound}, and suppose that
        \begin{equation}
             \frac{T}{\beta_{\Ti}^{2H}\Gamma_{\Ti}} \geq \frac{16L_r^2 \Hi^3C^{2\Hi}}{ \epsilon^2},
        \end{equation}
    for some fixed $\epsilon > 0$ and $H\geq 1$. Then, with probability at least $1-\delta$ after $T$ episodes, \alg achieves:
    \begin{equation}
        \min_{\piadv \in \Piadv} J(f, \hat{\pi}_\Ti, \piadv) \geq \min_{\piadv \in \Piadv} J(f, \pi^{\star}, \piadv) - \epsilon,
    \end{equation}
    where $\pifinal$ is the output of \alg, reported according to \cref{eq:rhucrl:output}, and $\pi^{\star}$ is the optimal robust policy given in \cref{eq:objective}. \looseness=-1

    \looseness=-1
    \end{restatable}
     
    \textbf{Gaussian Process Models.} We specialize the regret bound obtained in \cref{thm:exploration:regret:general_regret_bound} to the case of Gaussian Process (GP) models. 
    GPs are popular statistical models that are frequently used to model unknown dynamics \citep{deisenroth2011pilco,kamthe2018data,curi2020structured}. 
    These models are very expressive due to a versatility of possible kernel functions, and can naturally differentiate between aleatoric noise and epistemic uncertainty. 
    Moreover, GPs are known to be provably well-calibrated when the unknown dynamics $f$ are $B_f$-smooth as measured by the GP kernel. 
    
    In \cref{section:gp_models}, we recall the GP maximum information gain (MIG) which is a kernel-dependent quantity (first introduced by \citet{srinivas2010gaussian}), that is frequently used in various GP optimization works to characterize complexity of learning a GP model. Sublinear upper bounds for MIG are known (c.f.~ \citet{srinivas2010gaussian}) for most popularly used kernels (e.g., linear, squared-exponential, etc.), as well as for their compositions, e.g., additive kernels \cite{krause2011contextual}. We recall the known results and use MIG to express $\beta_T$ and upper bound $\Gamma_{\Ti}$ in \cref{thm:exploration:regret:general_regret_bound}.  For example, when we use independent GP models with either (i) linear or (ii) squared-exponential kernels, for every component, we obtain the following \emph{sublinear} (in $T$) regret bounds $O(\Hi^{3/2} \nstate \left[ (\nstate + \ninp + \ninpadv)\ln(\nstate \Ti \Hi) \right]^{(\Hi + 1)/ 2}  \sqrt{T} )$ and $O(\Hi^{3/2} \nstate \left[ \ln(\nstate \Ti \Hi) \right]^{(\nstate + \ninp + \ninpadv)(\Hi + 1)/ 2}  \sqrt{T} )$, respectively.
    
    Finally, we note that the previously used MIG bounds require $\X$ to be compact, which does not hold under the considered noise model in \Cref{as:reward_lipschitz}. By bounding the domain w.h.p., \citet{curi2020efficient} show that this only increases the MIG bounds (e.g., in case of the squared-exponential kernel) by at most a $\text{polylog}(T)$ factor.  


\section{Robust RL Applications and Experiments} \label{sec:experiments}\label{sec:applications}

\looseness -1 We now discuss concrete instantiations of \alg for three important robust RL scenarios: (i) {\em adversarial-robustness}, (ii) {\em action-robustness}, and (iii) {\em parameter-robustness}. 
In all of the above scenarios, we experimentally demonstrate that \alg outperforms or successfully competes with the state-of-the-art variants designed specifically for these settings. 
In \Cref{sec:extra_experiments}, we present experimental details and additional results. 


\textbf{Experimental Environments.} 
We use the Mujoco suite \citep{todorov2012mujoco} to demonstrate the effectiveness of our algorithms in all the considered robust-RL settings. 
In particular, we use the Half Cheetah, Hopper, Inverted Pendulum, Reacher, Swimmer, and Walker robots. 

    

\begin{figure*}[t]
  \centering\includegraphics[width=\textwidth]{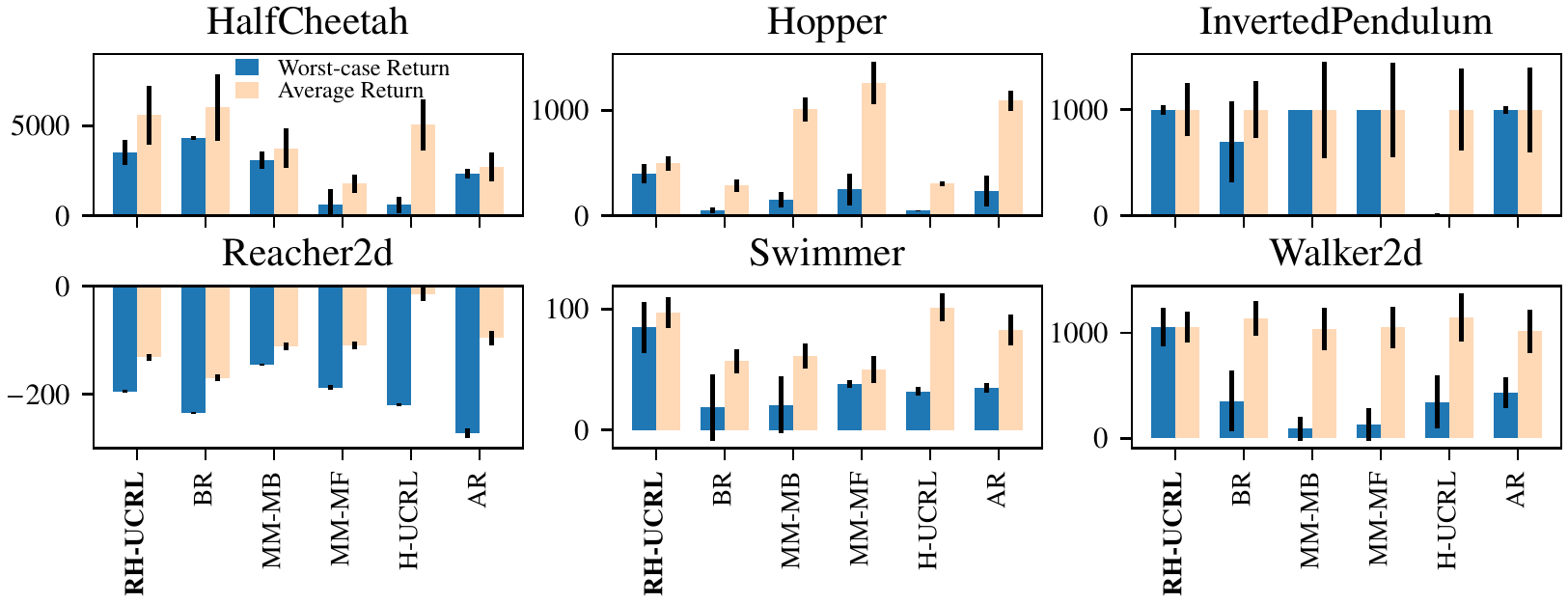}
    \vspace{-1em}
    
    \caption{
    Average and worst-case return of different algorithms in the Noisy Action-Robust Setting in Mujoco tasks. 
    \alg mostly outperforms other algorithms in terms of worst-case return. 
    The non-robust baseline, \hucrl, has good average performance but has an extreme drop in worst-case performance. Overall, the ablations perform better here than in the Adversarial-Robust setting.
  }
  \label{fig:action_robust_final}
\end{figure*}

\textbf{Baselines.} Besides the specific algorithms designed for each setting, we use \hucrl \citep{curi2020efficient} 
as a non-robust baseline and three ablations derived from \alg, namely \texttt{MiniMax}, \texttt{MiniMaxMF} and \texttt{BestResponse}. 
The \texttt{MiniMax} algorithm is:
\begin{equation}
    (\pi_\ti, \piadv_\ti) \in \argmax_{\pi \in \Pi} \min_{\piadv \in \Piadv} J^{(e)}_{\ti}(\pi, \piadv), \label{eq:minimax}
\end{equation}
where $J^{(e)}_{\ti}(\pi, \piadv)$ corresponds to the {\em expected} performance, where the expectation is taken with respect to the aleatoric and epistemic uncertainty, i.e., none of the players actively explore.
Next, the \texttt{MiniMaxMF} algorithm is a model-free implementation of \cref{eq:minimax} that uses SAC \citep{haarnoja2018soft} as the optimizer for each player.
The \texttt{BestResponse} algorithm is:
\begin{subequations}
    \begin{align}
        \pi_\ti &\in \argmax_{\pi \in \Pi} \min_{\piadv \in \Piadv} J_{\ti}^{(o)}(\pi, \piadv), \label{eq:best-response:learner} \\ 
        \piadv_\ti &\in \argmin_{\piadv \in \Piadv} J_{\ti}^{(e)}(\pi_\ti, \piadv) .  \label{eq:best-response:adversary}
    \end{align} \label{eq:best-response}%
\end{subequations}
Thus, the agent is the same as in \alg, whereas the adversary simply plays the best-response to the agent's policy and does not perform exploration with pessimism. 
The goal of \texttt{BestResponse} is to analyze if exploration of the adversary through pessimism is empirically important, of \texttt{MaxiMin-MB} is to analyze if any exploration is empirically important, and of \texttt{MaxiMin-MF} is to analyze if using a model of the dynamics is beneficial. 

\looseness=-1

\subsection{Adversarial-Robust Reinforcement Learning} \label{sseq:adversarial-rl}
This setting is the most general one that we also consider in \Cref{sec:RHUCRL}. The agent and the adversary can have distinct action spaces, which can also be seen as a particular instance of multi-agent RL with two competing agents. 
In the braking system motivating example, this can be used to model an adversarial state-dependent friction coefficient, e.g., icy roads. 
Having good robust performance in this setting implies braking robustly even with changing conditions.\looseness=-1

The deep robust RL algorithms that we compare with are \texttt{RARL} \citep{pinto2017robust} and \texttt{RAP} \citep{vinitsky2020robust}, and we use the adversarial action space proposed by \citet{pinto2017robust}.
We train all algorithms for 200 episodes except for \texttt{RARL} and \texttt{RAP}; since they are on-policy algorithms, we train them for 1000 episodes. 
To evaluate {\em robust performance} (recall \cref{eq:PAC}), we freeze the output policy and train only its adversary by using \texttt{SAC} for 200 episodes.

In \Cref{fig:adversarial_robust_final}, we show the {\em worst-case} and {\em average} returns on the different environments.
In terms of {\em average} performance, there is no algorithm that performs better than others in all of the environments.
On the other hand, comparing {\em worst-case} performance, \rhucrl clearly outperforms the robust ablations, deep robust RL and non-robust baselines. 
For example, in the Inverted Pendulum stabilization task, \alg is the {\em only} algorithm that discovers a robust policy while all other algorithms severely fail.   
\texttt{BestResponse} and \texttt{RAP} manage to learn a policy that stabilizes the pendulum even when they learn with an adversary. 
However, when facing a {\em worst-case} adversary, they fail to complete the task. 
\looseness=-1

Comparing \rhucrl with non-robust \hucrl, we see that in most environments it has comparable or better {\em worst-case} and {\em average} performance. 
This indicates that \rhucrl is not only robust, but using an adversary during training practically helps with exploration.  \citet{pinto2017robust} also report similar findings regarding robust training.
Comparing \rhucrl with the ablations, we see that \rhucrl achieves higher robust performance. From here, we conclude that exploring with both the agent and the adversary during training is crucial to achieve high robust performance in this setting. 
Finally, we see that both \texttt{RARL} and \texttt{RAP} have poor robust performance when trained for 1000 episodes, which demonstrates their sample inefficiency. 

\begin{figure*}[t]
   \centering\includegraphics[width=\textwidth]{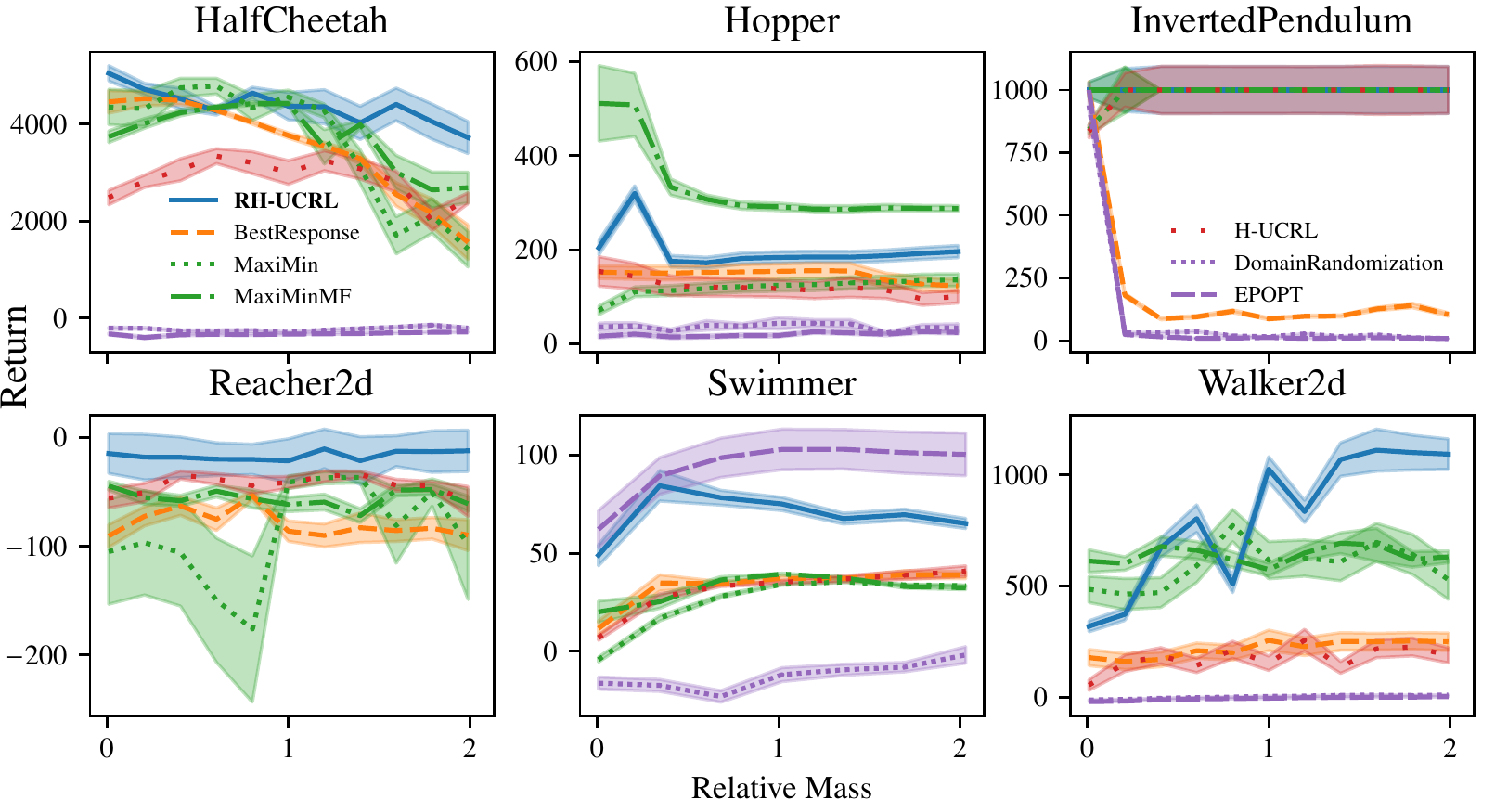}
    
    \caption{
    Returns of different algorithms in the Parameter-Robust Setting in Mujoco tasks for different masses during evaluation. 
    Although \alg optimizes for the worst-case relative mass in this setting, it also performs well over different value of mass parameters. 
  }
  \label{fig:parameter_robust_final}
\end{figure*}

\subsection{Action-Robust Reinforcement Learning} \label{sseq:action-robust-rl}
\citet{tessler2019action} introduce the action-robust setting, where both the agent and the adversary share the action space $\mathcal{\U}$ and jointly execute a single action in the environment. This is useful, e.g., to model robustness to changes in the actuator dynamics, e.g., due to tire wear or incorrect pressure in a braking system.
The action is sampled from a mixture policy $\u^{\text{mix}} \sim \pi^{\text{mix}} = \Gamma_\alpha(\pi, \piadv)$, where $\alpha \in [0,1]$ is a known parameter that controls the mixture proportion.
The system evolves according to $\x_{\hi+1} = f'(\x_\hi, \u^{\text{mix}}_\hi) + \noise_\hi$. 

Besides the previous baselines, we compare to \texttt{AR-DDPG} \cite{tessler2019action}, and show the results of the experiment in \Cref{fig:action_robust_final}. 
Here, \rhucrl is also comparable or better than the baselines in terms of {\em average} and {\em worst-case} returns. However, the ablations perform better than in the adversarial-robust setting. 
This is possibly due to the agent and adversary sharing the action space: The agent injects ``enough'' exploration to successfully learn both policies.\looseness=-1


\subsection{Parameter-Robust Reinforcement Learning} \label{sseq:domain-randomization-rl}

The goal in this setting is to be robust to changes in parameters, such as mass or friction, that can occur between training and test time. 
Being robust to a fixed parameter is equivalent to considering a stateless adversary policy in the \alg algorithm \eqref{eq:rhucrl}, i.e., $\Piadv: \varnothing \to \mathcal{\U}$.
Common benchmarks in this setting are domain randomization \citep{peng2018sim,tobin2017domain} and EP-OPT \cite{rajeswaran2016epopt}.
The former randomizes the parameters in the simulation and uses the {\em average} over these parameters as a surrogate of the maximum. 
The latter also randomizes the parameters but considers the CVaR as a surrogate of the maximum. 
As they are on-policy procedures, we train them using data for 1000 episodes. 
Finally, we evaluate the policies in different environments by varying the corresponding mass parameters.  \looseness=-1

We show the results of this setting in \Cref{fig:parameter_robust_final}. 
Although \alg optimizes for the worst-case parameter, it performs well over different mass parameter values, and, except in the Walker environment, its performance remains robust and nearly constant for different values of the mass parameter. 
\hucrl is trained with nominal mass only (relative mass = 1), and it suffers in performance when varying the mass. 
This is most notable in the Half Cheetah environment (see \cref{fig:parameter_robust_final}). 
The robust variants, instead, can alter the mass during training and often perform better than \hucrl. 
A particular case happens with the \texttt{BestResponse} algorithm in the Inverted Pendulum, where the adversary is greedy and so it swiftly chooses a small mass and never changes it during training. 
The agent learns only for this small mass and, when evaluated with different ones, it performs poorly. 
We also observe that in the Hopper, the \texttt{MaxiMin-MF} outperforms the \texttt{MaxiMin-MB}. 
The reason for this might be due to early stopping of the environment, as it is possible that the transitions collected in 200 episodes are not sufficient for learning the model, but allow for learning a policy in a model-free way. \looseness=-1

\section{Conclusion}
We introduced the \alg algorithm, a practical algorithm for deep model-based robust RL. It uses optimistic and pessimistic estimates of the robust performance to efficiently explore both the agent and fictitious adversary decision spaces during policy learning. We showed that \alg is provably robust and we established sample-complexity and regret guarantees. 
We instantiated our algorithm in important robust-RL settings such as adversarial-robust RL, parameter-robust RL, and action-robust RL. Empirically, \alg outperforms state-of-the-art deep robust RL algorithms. Perhaps surprisingly, its discovered robust policies often attain better non-robust performance than the ones found by non-robust algorithms, indicating benefits of \alg for exploration.  

\section*{Acknowledgments and Disclosure of Funding}
This project has received funding from the European Research Council (ERC) under the European Unions Horizon 2020 research and innovation program grant agreement No 815943. It was also supported by a fellowship from the Open Philanthropy Project.

\bibliographystyle{icml2020}
\bibliography{main}

\newpage
\onecolumn
\appendix
{\centering
    {\huge \bf Supplementary Material}

    {\Large \bf Combining Pessimism with Optimism for \\ Robust and Efficient Model-based Deep Reinforcement Learning  \\ [2mm] {\normalsize \bf {Submitted to ICML 2021} \par }  
}}

\section{Relevant Definitions and Results}

\begin{lemma}[Adapted from Corollary 1 in \citet{curi2020efficient}]
\label{lemma:one_step_diff}
    Based on \Cref{as:dynamics_f_lipschitz,as:model_predictions_lipschitz}, for every $\x, \x' \in \X$, it holds:
    \begin{equation} 
        \| f(\x, \pi(\x), \piadv(\x)) - \tilde{f}(\x', \pi(\x'), \piadv(\x')) \|_2 \leq L_f \sqrt{1 + L^2_{\pi} + L^2_{\piadv}} \|\x - \x' \|_2.
    \end{equation}
\end{lemma}

\begin{proof}
    \begin{subequations}
    \begin{align} 
        \| f(\x, \pi(\x), \piadv(\x)) - \tilde{f}(\x', \pi(\x'), \piadv(\x')) \|_2 & 
        \leq L_f \sqrt{ \|\x - \x' \|_2^2 + \|\pi(\x) - \pi(\x') \|_2^2 + \|\piadv(\x') - \piadv(\x)\|_2^2 }
        \label{eq:proof:lipschitz_f}\\ 
        &\leq \sqrt{ \|\x - \x' \|_2^2 + L^2_{\pi}\|\x - \x' \|_2^2 + L^2_{\piadv}\|\x - \x' \|_2^2}
         \label{eq:proof:lipschitz_pi} \\ 
        & = L_f \sqrt{1 + L^2_{\pi} + L^2_{\piadv}} \|\x - \x' \|_2.
    \end{align}
    \end{subequations}
    \cref{eq:proof:lipschitz_f} holds due to Lipschitz continuity of $f$ and \cref{eq:proof:lipschitz_pi} is due to Lipschitz continuity of $\pi$ and $\piadv$, which we assume in \Cref{as:dynamics_f_lipschitz,as:model_predictions_lipschitz}.
\end{proof}

\begin{lemma}[Adapted from Lemma 3 in \citet{curi2020efficient}]
\label{lemma:performance_diff}
    Based on \Cref{as:dynamics_f_lipschitz,as:model_predictions_lipschitz}, it holds:
    \begin{equation}
        |J(f, \pi, \piadv) - J(\tilde{f}, \pi, \piadv)| \leq L_r \sqrt{1 + L^2_{\pi} + L^2_{\piadv} } \sum_{\hi=0}^\Hi \E{\| \x_\hi - \tilde{\x}_\hi \|_2},
    \end{equation}
    where $\tilde{\x}_\hi$ for $\hi=0, \ldots, \Hi$ is the trajectory generated by the dynamics $\tilde{f}$, starting from $\tilde{\x}_0 = \x_0$ with $\omega_\hi = \tilde{\omega}_{\hi}$. 
\end{lemma}

\begin{proof}
    \begin{subequations}
    \begin{align} 
        |J(f, \pi, \piadv) - J(\tilde{f}, \pi, \piadv)| &=   \left|\E {\sum_{\hi=0}^\Hi r(\x, \u, \uadv)   - \sum_{\hi=0}^\Hi r(\tilde{\x}, \tilde{\u}, \tilde{\uadv})} \right| \label{eq:proof:def_J} \\
        & =  \left|\sum_{\hi=0}^\Hi \E{r(\x, \u, \uadv) - r(\x, \u, \uadv)} \right|  \label{eq:proof:exchange_E_sum} \\ 
        &\leq L_r \sqrt{1 + L^2_{\pi} + L^2_{\piadv} } \sum_{\hi=0}^\Hi \E{\| \x_\hi - \tilde{\x}_\hi \|_2} \label{eq:proof:reward_lipschitz}.
    \end{align}
    \end{subequations}
    \cref{eq:proof:def_J} follows by definition of $J$, \cref{eq:proof:exchange_E_sum} from linearity of expectation, and \cref{eq:proof:reward_lipschitz} from Lipschitzness of the policy and the reward function, which we assume in \Cref{as:dynamics_f_lipschitz}.
    
\end{proof}

The following lemma bounds the deviation between the optimistic/pessimistic and the true trajectory in a single episode.
\begin{lemma}[Adapted from Lemma 4 in \cite{curi2020efficient}]
\label{lemma:deviation_of_the_optimistic_pessimistic_trajectory}
    Under \Cref{as:dynamics_f_lipschitz,as:well_calibrated_model,as:model_predictions_lipschitz}, for all episodes $t \geq 1$, any $\eta \in [-1,1]$, $\hi \in \lbrace 1, \dots, H \rbrace$, $\pi \in \Pi$ and $\piadv \in \Piadv$ it holds: 
    \begin{equation}
        \| \x_{\hi,\ti} - \tilde{\x}_{\hi,\ti}\|_2 \leq 2 \beta_{\ti-1} \Big(1 + (L_f + 2 \beta_{\ti-1} L_{\sigma}) \sqrt{1 + L_{\pi}^2 + L_{\piadv}^2}\Big)^{\hi-1} \sum_{\hi'=0}^{\hi-1} \|\bsigma_{\ti-1}(\x_{\hi',\ti}, \pi_\ti(\x_{\hi',\ti}), \piadv_\ti(\x_{\hi',\ti})) \|_2, \label{eq:trajectory_difference_lemma}
    \end{equation}
    where $\tilde{\x}_{\hi,\ti}$ is generated by any system $\tilde{f} \in \mathcal{M}_\ti \coloneqq \left\{ \tilde{f} \text{ s.t. } | \tilde{f}(\x, \u, \uadv) - \bmu_{\ti-1}(\x, \u, \uadv) | \leq \beta_{\ti} \bsigma_{\ti-1}(\x, \u, \uadv)\right\}$. We refer to $\mathcal{M}_\ti$ as the set of plausible models at the beginning of episode $\ti$.
\end{lemma}
\begin{proof}
    To avoid notational clutter, we denote the closed-loop dynamics as $f^{\pi,\piadv}(\x) = f(\x, \pi(\x), \piadv(\x))$ and the closed-loop epistemic uncertainty as $\bsigma^{\pi,\piadv}(\x) = \bsigma(\x, \pi(\x), \piadv(\x))$. Likewise, we use the following Lipschitz constants shorthands $L_{f,\pi} \equiv L_f \sqrt{1 + L_{\pi}^2 + L_{\piadv}^2}$ and $L_{\sigma,\pi} \equiv L_\sigma \sqrt{1 + L_{\pi}^2 + L_{\piadv}^2}$. 
    \begin{subequations}
        
    We first prove by induction that 
    \begin{equation}
        \| \x_{\hi,\ti} - \tilde{\x}_{\hi,\ti}\|_2 \leq 2 \beta_{\ti-1} \sum_{\hi'=0}^{\hi-1}\left(L_{f,\pi} + 2 \beta_{\ti-1} L_{\sigma,\pi} \right)^{\hi-1 - \hi'} \| \bsigma_{\ti-1}^{\pi_{\ti},\piadv_{\ti}}(\x_{\hi',\ti}) \| \label{ineq:sim_inductive_hypothesis}
    \end{equation}
    For $\hi=0$, clearly $\x_{0,\ti} = \tilde{\x}_{0,\ti}$, while the right-hand-side of \cref{ineq:sim_inductive_hypothesis} is always non-negative. 
    We assume that for $\hi$ the inductive hypothesis \eqref{ineq:sim_inductive_hypothesis} holds. 
    For $\hi + 1$ we have:
    \begin{align}
    \| \x_{\hi + 1,\ti} - \tilde{\x}_{\hi + 1,\ti}\|_2 &= \|
    f^{\pi_\ti,\piadv_\ti}(\x_{\hi, \ti}) - \tilde{f}^{\pi_\ti,\piadv_\ti}(\tilde{\x}_{\hi, \ti})\|_2 \label{eq:sim_transition_dynamics} \\
    &= \| f^{\pi_\ti,\piadv_\ti}(\x_{\hi, \ti}) - \tilde{f}^{\pi_\ti,\piadv_\ti}(\tilde{\x}_{\hi, \ti}) + f^{\pi_\ti,\piadv_\ti}(\tilde{\x}_{\hi, \ti}) - f^{\pi_\ti,\piadv_\ti}(\tilde{\x}_{\hi, \ti})  \|_2 \label{eq:sim_add_substract_f} \\
    &\leq \|f^{\pi_\ti,\piadv_\ti}(\x_{\hi, \ti}) - f^{\pi_\ti,\piadv_\ti}(\tilde{\x}_{\hi, \ti}) \|_2 + \|f^{\pi_\ti,\piadv_\ti}(\tilde{\x}_{\hi, \ti}) - \tilde{f}^{\pi_\ti,\piadv_\ti}(\tilde{\x}_{\hi, \ti}) \|_2 \label{ineq:sim_triangular_f} \\
    & \leq L_{f,\pi} \| \x_{\hi,\ti} - \tilde{\x}_{\hi,\ti}\|_ 2 + \|f^{\pi_\ti,\piadv_\ti}(\tilde{\x}_{\hi, \ti}) - \tilde{f}^{\pi_\ti,\piadv_\ti}(\tilde{\x}_{\hi, \ti}) \|_2 \label{ineq:sim_lipshcitz_f} \\
    &\leq L_{f,\pi} \| \x_{\hi,\ti} - \tilde{\x}_{\hi,\ti}\|_ 2 
    + 2 \beta_{\ti-1} \| \bsigma^{\pi_\ti, \piadv_\ti}_{\ti-1}(\tilde{\x}_{\hi,\ti}) \|_2 \label{ineq:plausible_models} 
    \\  
    &= L_{f,\pi} \| \x_{\hi,\ti} - \tilde{\x}_{\hi,\ti}\|_2
    + 2 \beta_{\ti-1} \| \bsigma^{\pi_\ti, \piadv_\ti}_{\ti-1}(\tilde{\x}_{\hi,\ti}) + \bsigma^{\pi_\ti, \piadv_\ti}_{\ti-1}(\x_{\hi,\ti}) - \bsigma^{\pi_\ti, \piadv_\ti}_{\ti-1}(\x_{\hi,\ti})\|_2  \label{eq:sim_add_substract_sigma} \\ 
    & \leq L_{f,\pi} \| \x_{\hi,\ti} - \tilde{\x}_{\hi,\ti}\|_2 
    + 2 \beta_{\ti-1} \left(\| \bsigma^{\pi_\ti, \piadv_\ti}_{\ti-1}(\tilde{\x}_{\hi,\ti}) - \bsigma^{\pi_\ti, \piadv_\ti}_{\ti-1}(\x_{\hi,\ti})\|_2 + \| \bsigma^{\pi_\ti, \piadv_\ti}_{\ti-1}(\x_{\hi,\ti})\|_2  \right) \label{ineq:sim_triangular_sigma}\\ 
    &\leq \left(L_{f,\pi} + 2\beta_{\ti-1}L_{\sigma,\pi}\right) \| \x_{\hi,\ti} - \tilde{\x}_{\hi,\ti}\|_2 + 2 \beta_{\ti-1} \| \bsigma^{\pi_\ti, \piadv_\ti}_{\ti-1}(\x_{\hi,\ti})\|_2 \label{ineq:sim_lipshcitz_sigma} \\
    &\leq 2 \beta_{\ti-1} \sum_{\hi'=0}^{(\hi + 1) -1} \left(L_{f,\pi} + 2 \beta_{\ti-1} L_{\sigma,\pi} \right)^{(\hi + 1) - 1 - \hi'} \| \bsigma_{\ti-1}^{\pi_{\ti},\piadv_{\ti}}(\x_{\hi',\ti}) \|_2 \label{ineq:sim_apply_inductive_hypothesis}
    \end{align}
    \end{subequations}
    Here, \cref{eq:sim_transition_dynamics} holds by applying the transition dynamics $f^{\pi_\ti, \piadv_\ti}$ and $\tilde{f}^{\pi_\ti, \piadv_\ti}$ with the same noise realization $\omega_\hi = \tilde{\omega}_\hi$;  \cref{eq:sim_add_substract_f} holds by adding and subtracting $f^{\pi_\ti, \piadv_\ti}(\tilde{\x}_{\hi,\ti})$; \cref{ineq:sim_triangular_f} follows from the triangular inequality;  \cref{ineq:sim_lipshcitz_f} comes from \cref{lemma:one_step_diff};  \cref{ineq:plausible_models} holds due to both $f$ and $\tilde{f}$ belonging to the set of plausible models $\mathcal{M}_t$ (see \cref{lemma:deviation_of_the_optimistic_pessimistic_trajectory} for its definition);  \cref{eq:sim_add_substract_sigma} holds by adding and substracting $\bsigma^{\pi_\ti, \piadv_\ti}(\x_{\hi,\ti})$; \cref{ineq:sim_triangular_sigma} is by applying the triangular inequality once more; \cref{ineq:sim_lipshcitz_sigma} is due to the Lipschitz continuity of $\bsigma$ as per \cref{as:model_predictions_lipschitz}; and \eqref{ineq:sim_apply_inductive_hypothesis} holds by replacing the inductive hypothesis from \eqref{ineq:sim_inductive_hypothesis}. 
    
    Finally, we notice that $\left(L_{f,\pi} + 2 \beta_{\ti-1} L_{\sigma,\pi} \right)^{\hi - 1 - \hi'} < \left(1 + L_{f,\pi} + 2 \beta_{\ti-1} L_{\sigma,\pi} \right)^{\hi - 1 - \hi'} \leq \left(1 + L_{f,\pi} + 2 \beta_{\ti-1} L_{\sigma,\pi} \right)^{\hi-1}$ and the main result follows by combining this with \cref{ineq:sim_apply_inductive_hypothesis}. 
    
    \end{proof}

\section{Proofs from \Cref{sec:theoretical_analysis}}
\label{sec:proofs_main}
We start the analysis of the performance of \alg, by first bounding its instantaneous robust-regret by the difference between optimistic and pessimistic performance estimates.  
\begin{lemma}
  \label{lem:instantaneous_regret_bound}
  Let $\pi^{\star}$ be the benchmark policy from \cref{eq:objective}, and let $\pi_\ti$ and $\piadv_\ti$ be the policies selected by \alg at time $t$. Under the callibrated model \Cref{as:well_calibrated_model}, the following holds with probability at least $1 - \delta$:   
  \begin{equation}
    \min_{\piadv \in \Piadv} J(f, \pi^\star, \piadv) - \min_{\piadv \in \Piadv} J(f, \pi_\ti, \piadv) \leq J_\ti^{(o)}(\pi_\ti, \piadv_\ti) - J^{(p)}(\pi_\ti, \piadv_\ti).
  \end{equation}
   
\end{lemma}

\begin{proof}
    We refer to the considered quantity $\min_{\piadv \in \Piadv} J(f, \pi^\star, \piadv) - \min_{\piadv \in \Piadv} J(f, \pi_\ti, \piadv)$ as the robust instantaneous regret of the selected policy $\pi_\ti$, and we proceed by providing its upper bound:
    \begin{subequations}
    \begin{align}
    \min_{\piadv \in \Piadv} J(f, \pi^\star, \piadv) - \min_{\piadv \in \Piadv} J(f, \pi_\ti, \piadv) & \leq \min_{\piadv \in \Piadv} J_\ti^{(o)}(\pi^\star, \piadv) - \min_{\piadv \in \Piadv} J(f, \pi_\ti, \piadv)\label{seq:lemma1:robust_regret}\\
    & \leq \min_{\piadv \in \Piadv} J_\ti^{(o)}(\pi_\ti, \piadv) - \min_{\piadv \in \Piadv} J(f, \pi_\ti, \piadv) \label{seq:lemma1:RHUCRL:protagonist}  \\
    & \leq J_\ti^{(o)}(\pi_\ti, \piadv_\ti) - \min_{\piadv \in \Piadv} J(f, \pi_\ti, \piadv)  \label{seq:lemma1:minimum} \\
    & \leq J_\ti^{(o)}(\pi_\ti, \piadv_\ti) - \min_{\piadv \in \Piadv} J^{(p)}(\pi_\ti, \piadv)  \label{seq:lemma1:LCB}  \\
    &= J_\ti^{(o)}(\pi_\ti, \piadv_\ti) - J^{(p)}(\pi_\ti, \piadv_\ti). \label{seq:lemma1:RHUCRL:antagonist} 
    \end{align}
    \end{subequations}
    Here, inequality \eqref{seq:lemma1:robust_regret} holds by definition of the optimistic estimate in \cref{eq:optimistic_performance}; inequality \eqref{seq:lemma1:RHUCRL:protagonist} holds by definition of protagonist policy in the \rhucrl algorithm \eqref{eq:rhucrl:learner}; and inequality \eqref{seq:lemma1:LCB} holds by definition of the pessimistic estimate in \cref{eq:pessimistic_performance}; finally, equality \eqref{seq:lemma1:RHUCRL:antagonist} holds by definition of the antagonist policy in the \rhucrl algorithm \eqref{eq:rhucrl:adversary}.
\end{proof}

\begin{lemma}
  \label{lem:performance_difference_bound}
    Under \Cref{as:dynamics_f_lipschitz,as:well_calibrated_model,as:model_predictions_lipschitz}, let $\pi_\ti$ and $\piadv_\ti$ be the policies selected by \rhucrl at episode $t$. Then, the following holds for the difference between its optimistic and pessimistic performance:
    \begin{equation}
       J_\ti^{(o)}(\pi_\ti, \piadv_\ti) - J^{(p)}(\pi_\ti, \piadv_\ti)
        \leq 4 L_r \beta_T^H C^H \sum_{\hi=0}^{\Hi} \E{ \sum_{\hi'=0}^{\hi-1} 
        \|\bsigma_{\ti-1}(\x_{\hi',\ti}, \pi_\ti(\x_{\hi',\ti}), \piadv_\ti(\x_{\hi',\ti})) \|_2},
    \end{equation}
    where $C= (1 + L_f + L_{\sigma})(1 + L_{\pi}^2 + L_{\piadv}^2)^{1/2}$. 
\end{lemma}

\begin{proof}
    \begin{subequations}
    \begin{align}
        J_\ti^{(o)}(\pi_\ti, \piadv_\ti) - J^{(p)}(\pi_\ti, \piadv_\ti)
        &\leq \left| J_\ti^{(o)}(\pi_\ti, \piadv_\ti) - J(f, \pi_\ti, \piadv_\ti) \right|  + \left| J_\ti^{(p)}(\pi_\ti, \piadv_\ti) - J(f, \pi_\ti, \piadv_\ti) \right| \label{seq:lemma3:triangular} \\ 
        & \leq L_r \sqrt{1+L^2_\pi + L^2_{\piadv}} \sum_{\hi=0}^\Hi \left( \E{\| \x_{\hi,\ti} - \x_{\hi,\ti}^{(o)} \|_2}+ \E{\| \x_{\hi,\ti} - \x_{\hi,\ti}^{(p)} \|_2} \right) \label{seq:lemma2:lemma3HUCRL}
    \end{align} Here, \cref{seq:lemma3:triangular} holds by the triangle inequality and \cref{seq:lemma2:lemma3HUCRL} follows from Lemma~\ref{lemma:performance_diff}.

    We proceed to upper bound terms $\| \x_{\hi,\ti} - \x_{\hi,\ti}^{(o)}\|_2$ and $\| \x_{\hi,\ti} - \x_{\hi,\ti}^{(p)} \|_2$.  
    From \Cref{lemma:deviation_of_the_optimistic_pessimistic_trajectory}, it follows that both terms can be bounded in the same way as follows:
    \begin{equation}
        \| \x_{\hi,\ti} - \x_{\hi,\ti}^{(o)}\|_2 \leq 2 \beta_{\ti-1} \Big((1 + L_f + 2 \beta_{\ti-1} L_{\sigma}) \sqrt{1 + L_{\pi}^2 + L_{\piadv}^2}\Big)^{\hi-1} \sum_{\hi'=0}^{\hi-1} \|\bsigma_{\ti-1}(\x_{\hi',\ti}, \pi_\ti(\x_{\hi',\ti}), \piadv_\ti(\x_{\hi',\ti})) \|_2,
    \end{equation}
    as $f^{(o)}$ and $f^{(p)}$ belong to the set of plausible models $\mathcal{M}_t$ (from \cref{lemma:deviation_of_the_optimistic_pessimistic_trajectory}). 

    By applying the previous bound twice in \cref{seq:lemma2:lemma3HUCRL}, and by denoting
    $$C\coloneqq (1 + L_f + 2L_{\sigma})(1 + L_{\pi}^2 + L_{\piadv}^2)^{1/2},$$
    we arrive at:
    \begin{equation}
       J_\ti^{(o)}(\pi_\ti, \piadv_\ti) - J^{(p)}(\pi_\ti, \piadv_\ti)
        \leq 4 L_r \beta_T^H C^H \sum_{\hi=0}^{\Hi} \E{ \sum_{\hi'=0}^{\hi-1} 
        \|\bsigma_{\ti-1}(\x_{\hi',\ti}, \pi_\ti(\x_{\hi',\ti}), \piadv_\ti(\x_{\hi',\ti})) \|_2},
    \end{equation}
    where we used $\ti \leq \Ti$ and $1 \leq \beta_\ti$ is non-decreasing in $\ti$. 
    \end{subequations}

\end{proof}

\generalregretbound*
\begin{proof}[Proof of Theorem 1]
    We bound the robust cumulative regret as follows:
    \begin{subequations}
    \begin{align}
        R_{\Ti} &= \sum_{\ti=1}^{\Ti} \underbrace{\min_{\piadv \in \Piadv} J(f, \pi^\star, \piadv) - \min_{\piadv \in \Piadv} J(f, \pi_\ti, \piadv_\ti)}_{\coloneqq r_\ti} \\
        &\leq \sqrt{\Ti \sum_{\ti=1}^\Ti r_\ti^2} \label{seq:thm1:cs} \\
        &\leq \sqrt{\Ti \sum_{\ti=1}^\Ti (4 L_r \beta_T^H C^H)^2 \Bigg( \sum_{\hi=0}^{\Hi} \E{\sum_{\hi'=0}^{\hi-1} \|\bsigma_{\ti-1}(\x_{\hi',\ti}, \pi_\ti(\x_{\hi',\ti}), \piadv_\ti(\x_{\hi',\ti})) \|_2 }\Bigg)^2}  \label{seq:thm1:lemma3and4}\\
        & = 4 L_r \beta_T^H C^H \sqrt{\Ti} \sqrt{ \sum_{\ti=1}^\Ti  \Bigg( \sum_{\hi=0}^{\Hi} \E{\sum_{\hi'=0}^{\hi-1} \|\bsigma_{\ti-1}(\x_{\hi',\ti}, \pi_\ti(\x_{\hi',\ti}), \piadv_\ti(\x_{\hi',\ti})) \|_2 }\Bigg)^2} \label{seq:thm1:monotonicity_of_L_HT} \\
        &\leq 4 L_r \beta_T^H C^H \Hi \sqrt{\Ti} \sqrt{\sum_{\ti=1}^\Ti \Bigg(  \E{ \sum_{\hi'=0}^{\Hi} \|\bsigma_{\ti-1}(\x_{\hi',\ti}, \pi_\ti(\x_{\hi',\ti}), \piadv_\ti(\x_{\hi',\ti})) \|_2  } \Bigg)^2} \\
        &\leq 4 L_r \beta_T^H C^H \Hi \sqrt{\Ti} \sqrt{\sum_{\ti=1}^\Ti   \E{ \Bigg( \sum_{\hi'=0}^{\Hi} \|\bsigma_{\ti-1}(\x_{\hi',\ti}, \pi_\ti(\x_{\hi',\ti}), \piadv_\ti(\x_{\hi',\ti})) \|_2 \Bigg)^2 }}\label{seq:thm1:jensen}\\
        &\leq 4 L_r \beta_T^H C^H \Hi^{3/2} \sqrt{\Ti} \sqrt{\sum_{\ti=1}^\Ti   \E{  \sum_{\hi'=0}^{\Hi} \|\bsigma_{\ti-1}(\x_{\hi',\ti}, \pi_\ti(\x_{\hi',\ti}), \piadv_\ti(\x_{\hi',\ti})) \|_2^2  }} \label{seq:thm1:cs_2} \\
        &\leq 4 L_r \beta_T^H C^H {\Hi}^{3/2} \sqrt{\Ti \Gamma_T},\label{seq:thm1:final}
        %
    \end{align}
    \end{subequations}
    where \cref{seq:thm1:cs} is due to the Cauchy-Schwarz's inequality;  \cref{seq:thm1:lemma3and4} is due to \Cref{lem:instantaneous_regret_bound} and \Cref{lem:performance_difference_bound}. Finally, \cref{seq:thm1:jensen} follows from the Jensen's inequality, \cref{seq:thm1:cs_2} follows from Cauchy-Schwarz's inequality, and \cref{seq:thm1:final} follows from the definition of $\Gamma_{\Ti}$.
    
\end{proof}

\simpleregretbound*
\begin{proof}[Proof of Corollary 1]
    We start the proof by recalling some of the previously obtained results.
    The instantaneous regret $r_\ti(\pi_\ti)$ of a policy $\pi_\ti$ selected at episode $\ti$ in \cref{eq:rhucrl:learner} is given by:
    \begin{equation}
        r_\ti(\pi_\ti) = \min_{\piadv \in \Piadv} J(f, \pi^\star, \piadv) - \min_{\piadv \in \Piadv} J(f, \pi_\ti, \piadv).
    \end{equation}
    From \Cref{lem:instantaneous_regret_bound} and \Cref{lem:performance_difference_bound}, it follows that 
    \begin{equation} \label{sec:corollary1:instant_regret_bound}
        r_\ti(\pi_\ti) \leq 4 L_r \beta_T^H C^H \sum_{\hi=0}^{\Hi} \E{ \sum_{\hi'=0}^{\hi-1} 
        \|\bsigma_{\ti-1}(\x_{\hi',\ti}, \pi_\ti(\x_{\hi',\ti}), \piadv_\ti(\x_{\hi',\ti})) \|_2}.
    \end{equation}
    We also define
    \begin{equation}
        \bar{r}(\pi_\ti) \coloneqq \min_{\piadv \in \Piadv} J(f, \pi^\star, \piadv) - \min_{\piadv \in \Piadv} J^{(p)}(\pi_\ti, \piadv),
    \end{equation}
    and note that $r(\pi_\ti) \leq \bar{r}(\pi_\ti)$ for every $\pi_\ti$, since $\min_{\piadv \in \Piadv} J^{(p)}(\pi_\ti, \piadv) \leq \min_{\piadv \in \Piadv} J(f,\pi_\ti, \piadv)$. Another useful observation is that the same bound obtained in \Cref{sec:corollary1:instant_regret_bound} also holds in case of $\bar{r}(\pi_\ti)$, i.e., 
    \begin{equation} \label{sec:corollary1:regret_bar_upper_bound}
        r(\pi_\ti) \leq \bar{r}(\pi_\ti) \leq 4 L_r \beta_T^H C^H \sum_{\hi=0}^{\Hi} \E{ \sum_{\hi'=0}^{\hi-1} 
        \|\bsigma_{\ti-1}(\x_{\hi',\ti}, \pi_\ti(\x_{\hi',\ti}), \piadv_\ti(\x_{\hi',\ti})) \|_2}.
    \end{equation}
    
    Recall, that the reported policy $\hat{\pi}_\Ti$ from \cref{eq:rhucrl:output} is chosen among the previously selected episodic policies $\lbrace \pi_1, \dots, \pi_\Ti\rbrace$, such that 
    \begin{equation} \label{sec:corollary1:reported_policy}
        \hat{\pi}_\Ti = \argmin_{\pi \in \lbrace \pi_1, \dots, \pi_\Ti\rbrace} \bar{r}(\pi).
    \end{equation}
    It follows that:
    \begin{subequations}
        \begin{align}
            r(\hat{\pi}_\Ti) &\leq \bar{r}(\hat{\pi}_\Ti) \label{sec:corollary1:regret_relation} \\
            &\leq \frac{1}{\Ti} \sum_{\ti=1}^{\Ti} \bar{r}(\pi_\ti) \label{sec:corollary1:min_smaller_than_avg}\\
            &\leq \frac{1}{\Ti} \sum_{\ti=1}^{\Ti} 4 L_r \beta_T^H C^H \sum_{\hi=0}^{\Hi} \E{ \sum_{\hi'=0}^{\hi-1} 
        \|\bsigma_{\ti-1}(\x_{\hi',\ti}, \pi_\ti(\x_{\hi',\ti}), \piadv_\ti(\x_{\hi',\ti})) \|_2} \label{sec:corollary1:instant_regret_bound_applied} \\
            &\leq \frac{1}{\Ti} 4 L_r \beta_T^H C^H H \sum_{\ti=1}^{\Ti} \E{ \sum_{\hi'=0}^{\Hi}
        \|\bsigma_{\ti-1}(\x_{\hi',\ti}, \pi_\ti(\x_{\hi',\ti}), \piadv_\ti(\x_{\hi',\ti})) \|_2} \label{sec:corollary1:upper_bound_h}\\
        &\leq \frac{1}{\Ti} 4 L_r \beta_T^H C^H H \sqrt{T} \sqrt{\sum_{\ti=1}^{\Ti}  \E{ \left(\sum_{\hi'=0}^{\Hi}
        \|\bsigma_{\ti-1}(\x_{\hi',\ti}, \pi_\ti(\x_{\hi',\ti}), \piadv_\ti(\x_{\hi',\ti})) \|_2 \right)^2}  } \label{sec:corollary1:upper_bound_CS} \\
        &\leq \frac{1}{\Ti} 4 L_r \beta_T^H C^H H \sqrt{T} \sqrt{\sum_{\ti=1}^{\Ti} \E{ \sum_{\hi'=0}^{\Hi}
        \|\bsigma_{\ti-1}(\x_{\hi',\ti}, \pi_\ti(\x_{\hi',\ti}), \piadv_\ti(\x_{\hi',\ti})) \|_2^2}  } \label{sec:corollary1:upper_bound_Jensen} \\
            &\leq \frac{4 L_r \beta_T^H C^H\Hi^{3/2} \sqrt{\Ti \Gamma_{\Ti}}}{\Ti}  \label{sec:corollary1:cumulative_regret_bound}
        \end{align}
    \end{subequations}
    where \cref{sec:corollary1:regret_relation} follows from \cref{sec:corollary1:regret_bar_upper_bound}, and \cref{sec:corollary1:min_smaller_than_avg} follows from the policy reporting rule in \cref{sec:corollary1:reported_policy} and by upper bounding minimum with average. Finally, \cref{sec:corollary1:instant_regret_bound_applied} is due to \cref{sec:corollary1:regret_bar_upper_bound}, and \cref{sec:corollary1:upper_bound_h,sec:corollary1:upper_bound_CS,sec:corollary1:upper_bound_Jensen,sec:corollary1:cumulative_regret_bound} follow the same argument as in the proof of  \cref{thm:exploration:regret:general_regret_bound}.  

    To achieve $r(\hat{\pi}_\Ti) \leq \epsilon$ for some given $\epsilon > 0$, we require that
    $$
    \frac{4 L_r \beta_T^H C^H \Hi^{3/2} \sqrt{\Ti \Gamma_{\Ti}}}{\Ti} \leq \epsilon.$$
    By simple inversion it follows that we require the following number of episodes $\Ti$:
    $$
        \frac{\Ti}{\beta_T^{2\Hi}  \Gamma_{\Ti}} \geq \frac{16 L_r^2 \Hi^3 C^{2\Hi}}{\epsilon^2}
    $$
    to achieve $r(\hat{\pi}_\Ti) \leq \epsilon$.
\end{proof}

\newpage

\section{Gaussian Process Dynamical Models}
\label{section:gp_models}

In this section, we formalize the setting in which the true dynamics $f$ in \cref{eq:stochastic_dynamic_system_additive} has bounded norm in an RKHS induced by a continuous, symmetric positive definite kernel function $k: \mathcal{Z} \times \mathcal{Z} \to \mathbb{R}$, with $\mathcal{Z} = \X \times \U \times \Uadv$. We denote by $\mathcal{K}$ the corresponding RKHS.
Having a norm $\|f\|_{\mathcal{K}} \leq B_f$ for some finite $B_f > 0$ means that the RKHS is well-suited for capturing $f$ \citep{durand2018streaming}. 

Due to the episodic nature of the problem, we follow the batch analysis from \citet{desautels2014parallelizing} and generalize it to the MDP setting with multiple outputs. 
In particular, we observe $\Hi$ transitions per episode and at the beginning of each episode we use the model to make decisions for other $\Hi$ steps. 
To extend to multiple outputs we build $\nstate$ copies of the dataset such that $\dataset_{1:\ti, i} = \left\{ (\x_{\ti', \hi}, \u_{\ti', \hi}, \uadv_{\ti', \hi}), \x_{\ti', \hi + 1, i} \right\}_{\hi=0, \ti'=1}^{\Hi-1, \ti}$, each with $\ti\Hi$ transitions. I.e., the $i$-th dataset has as covariates the state-action-adversarial action and as target the $i$-th coordinate of the next-state. We denote the covariates $z_{\ti,\hi} \equiv (\x_{\ti, \hi}, \u_{\ti, \hi}, \uadv_{\ti, \hi})$ and the targets as $y_{\ti,\hi,i} \equiv \x_{\ti', \hi + 1, i}$.  
Finally, we build $\nstate$ models as 
\begin{subequations} \label{eq:GP-Regression}
\begin{align}
    \mu_\ti(z,i) &= k_\ti(z)^\top (K_\ti + \lambda I)^{-1} y_{1:\Hi\ti,i}, \label{eq:posterior_mean}\\
    k_{\ti}(z, z', i) &= k(z, z') -  k_\ti(z)^\top (K_\ti + \lambda I)^{-1}k_\ti(z'), \label{eq:posterior_covariance} \\ 
    \sigma_{\ti}^2(z, i) &= k_\ti(z, z), \label{eq:posterior_variance}
\end{align}
where $\x'_{1:\Hi\ti,i}$ is the column vector of the $i$-th coordinate of all the next-states in the dataset, $K_{\ti}$ is the kernel matrix, $I$ is the identity matrix of appropriate dimensions and we use $\lambda = \nstate \Hi$ as the same data is used in all the $\nstate$ models. 

Stacking together the posterior mean and variance into column vectors we get:
\begin{align}
    \bmu_{\ti}(z) &= \left[\mu_\ti(z,1), \ldots, \mu_{\ti}(z, \nstate) \right]^\top \label{eq:posterior_mean_vector},\\
    \bsigma_{\ti}(z) &= \left[\sigma_{\ti}^2(z, 1), \ldots, \sigma_{\ti}^2(z, \nstate) \right]^\top. \label{eq:posterior_variance_vector}
\end{align}
\end{subequations}

A key quantity that we consider in this work is the \emph{(maximum) information gain} that measures the information about the true dynamics $f$ by observing $n$ transitions. 

\begin{definition}[Information Gain \citep{cover1991entropy,srinivas2010gaussian,durand2018streaming}]
    The information gain is the mutual information between the true function $f$ and a set of observations at locations $Z$ and is the difference between the entropy of such observations and the conditional entropy of the observations given the funciton values i.e., 
    \begin{subequations}
    
    \begin{equation}
        I(f_Z; y_{Z}) = H(y_{Z}) - H(y_{Z} | f_Z), \label{eq:info_gain}
    \end{equation}
    where $f_Z$ is the noise-free evaluation of $f$ at $Z$ and $y_Z$ is the noisy observation. In the case of GP models as in \cref{eq:GP-Regression}, the information gain is:
    \begin{equation}
        I(f_Z; y_Z) = \frac{1}{2} \sum_{k=1}^n \ln(1 + \lambda^{-1} \sigma^2_{k-1}(z_k)). \label{eq:info_gain_gp}
    \end{equation}
    \end{subequations}

\end{definition}

Next, we introduce the maximum information gain, which is a parameter that quantifies how hard the learning problem is and tightly upper bounds the {\em effective-dimensionality} of the problem \citep{valko2013finite}.

\begin{definition}[Maximum Information Gain \citep{srinivas2010gaussian}]
    The maximum information gain is the maximum of the information gain, taken over all datasets with a fixed size $n$, i.e., 
    \begin{subequations}
    \begin{equation}
        \gamma_{n}(k; Z) \coloneqq \max_{Z \subset \mathcal{Z}, |Z| = n  } I(f_z; y_{z}).
    \end{equation}
    In the particular case of GP models, this reduces to:
    \begin{equation}
        \gamma_{n}(f; Z) = \max_{ \left\{z_1, \ldots, z_n\right\} \subset \mathcal{Z}} \frac{1}{2} \sum_{k=1}^n \ln(1 + \lambda^{-1} \sigma^2_{k-1}(z_k)).
    \end{equation}
    \end{subequations}
\end{definition}

\citep{srinivas2010gaussian} show that the Maximum Information Gain (MIG) is sub-linear in the number of observations for commonly used kernels. 
The main idea now is to bound the complexity measure $\Gamma_\Ti$ defined in \Cref{eq:complexity_measure} in terms of the MIG and, for commonly used kernels, we achieve no-regret algorithms. Towards this end, we recall two results related to GP-models in \Cref{lemma:posterior_variance}.

\begin{lemma}{Posterior variance bound \citep{chowdhury2019online}} \label{lemma:posterior_variance}
Let $k: \mathcal{Z} \times \mathcal{Z} \to \mathbb{R}$ be a symmetric positive semi-definite kernel with bounded variance, i.e., $k(z, z) \leq 1, \forall z \in \mathcal{Z}$ and $f \sim GP_{\mathcal{Z}}(0, k)$ be a sample from the associated Gaussian process, then for all $n \geq 1$ and $z \in \mathcal{Z}$: 
\begin{subequations}
\begin{align}
        \sigma_{n-1}^2(z) &\leq (1 + \lambda^{-1}) \sigma_{n}^2(z), \label{eq:posterior_variance_bound} \\ 
        \sum_{k=1}^n \sigma_{k-1}^2(z_k) &\leq (1 + 2\lambda) \sum_{k=1}^n \frac{1}{2} \ln \left[ 1 + \lambda^{-1}\sigma_{k-1}(z_k) \right] = (1 + 2\lambda) I(f_Z;y_Z). \label{eq:sum_posterior_variance_bound}
\end{align}
\end{subequations}
\end{lemma}
\begin{proof}
    See \citep[Lemma 2]{chowdhury2019online}
\end{proof}

Although the left-hand-side in \cref{eq:sum_posterior_variance_bound} has the flavor of the complexity measure $\Gamma_\Ti$ defined in \Cref{eq:complexity_measure} it is not exactly the same as we only update the posterior once every $\Hi\nstate$ observations. 
This is related to the batch setting analyzed in \citet{desautels2014parallelizing}. 
The next lemma bounds the sum of posterior variances in terms of the information gain.

\begin{lemma}{Complexity measure $\Gamma_\Ti$ is upper bounded by MIG}
Let $k: \mathcal{Z} \times \mathcal{Z} \to \mathbb{R}$ be a symmetric positive semi-definite kernel with bounded variance, i.e., $k(z, z) \leq 1, \forall z \in \mathcal{Z}$ and $f \sim GP_{\mathcal{Z}}(0, k)$ be a sample from the associated Gaussian process, then for all $\ti \geq 1$ and $z \in \mathcal{Z}$ for the GP-model given in \cref{eq:GP-Regression} with $\lambda=\Hi\nstate$ we have that:
\begin{equation}
    \Gamma_\ti \leq  2e\nstate\Hi \gamma_{\nstate \Hi \ti}(k, \mathcal{Z}) \label{eq:complexity_measure_bounded_MIG}
\end{equation}

\end{lemma}

\begin{proof}
The proof is based on \citet[Lemma 11]{chowdhury2019online} and adapted to our setting. 
\begin{subequations}
\begin{align}
    \sum_{\ti=1}^{\Ti} \sum_{(\x,\u, \uadv) \in \tilde{\dataset}_{t} } \|\bsigma_{\ti-1}(\x, \u, \uadv)\|^2_2 &=
    \sum_{\ti'=1}^{\ti} \sum_{\hi=0}^{\Hi-1} \sum_{i=1}^{\nstate} \sigma^2_{(\ti'-1)\Hi\nstate} (z_{\ti', \hi, i}) \label{eq:2-norm} \\
    &\leq \sum_{\ti'=1}^{\ti} \sum_{\hi=0}^{\Hi-1} \sum_{i=1}^{\nstate} (1 + \lambda^{-1})^{\nstate \hi + i - 1 }  \sigma^2_{(\ti'-1)\Hi\nstate + \hi\nstate + i} (z_{\ti', \hi, i}) \label{eq:apply_posterior_variance_bound} \\
    &\leq (1 + \lambda^{-1})^{\nstate (\Hi-1) + \nstate - 1 }  \sum_{\ti'=1}^{\ti} \sum_{\hi=0}^{\Hi-1} \sum_{i=1}^{\nstate} \sigma^2_{(\ti'-1)\Hi\nstate + \hi\nstate + i} (z_{\ti', \hi, i})  \label{eq:increasing_terms} \\
    & \leq (1 + \lambda^{-1})^{\nstate \Hi - 1 } (2 \lambda + 1) I(f_Z;y_Z) \label{eq:apply_sum_posterior_variance_bound} \\
    &\leq  2e\nstate\Hi I(f_Z;y_Z) \label{eq:tricks}
\end{align}

Here, equality \eqref{eq:2-norm} is the definition of the 2-norm; inequality \eqref{eq:apply_posterior_variance_bound} is due to \cref{eq:posterior_variance_bound} in \Cref{lemma:posterior_variance}; inequality \eqref{eq:increasing_terms} is due to $1 + \lambda^{-1} \geq 1$; inequality \eqref{eq:apply_sum_posterior_variance_bound} is due to \cref{eq:sum_posterior_variance_bound} in \Cref{lemma:posterior_variance}; finally the last inequality \eqref{eq:tricks} is due to $(1 + \lambda^{-1})^\lambda \leq e$ and $(1 + \lambda^{-1})^{-1} (2\lambda + 1) \leq 2\lambda$. 
The statement follows by taking the maximum over data sets. 
\end{subequations}
\end{proof}

Next, we will show that GP models are calibrated and satisfy \Cref{as:well_calibrated_model}. 

\begin{lemma}{Concentration of an RKHS member \citep[Theorem 1]{durand2018streaming}} \label{lemma:concentration}
Given \Cref{as:dynamics_f_lipschitz}, $\|f\|_{\mathcal{K}} \leq B_f$, and $k(\cdot, \cdot) \leq 1$, then, for all $\delta \in [0, 1]$, which probability at least $1-\delta$, it holds simultaneously over all $z \in Z$ and $t \geq 0$, 

\begin{equation}
    | f(z) - \mu_{\ti}(z) | \leq \left(B_f + \frac{\sigma}{\lambda}\sqrt{2 \ln(1 / \delta) + 2 \gamma_{\ti}}\right) \sigma_t(z),
\end{equation}
where $\mu_{\ti}(z)$ and $\sigma_t(z)$ are given by \cref{eq:posterior_mean} and \cref{eq:posterior_variance}.

\end{lemma}

Thus we know that, using $\beta_{\ti} = \left(B_f + \frac{\sigma}{\lambda}\sqrt{2 \ln(1 / \delta) + 2 \gamma_{\ti}}\right)$, \Cref{as:dynamics_f_lipschitz} holds for a single dimension. The extension to multiple dimensions is straightforward and has been done by \citep[Lemma 10]{chowdhury2019online} and \citep[Lemma 11]{curi2020efficient}, using $\lambda \gets \Hi \nstate$ and $t \gets \ti \Hi \nstate$.




Putting together results of the previous sections, we know by \Cref{lemma:concentration} that, under \Cref{as:dynamics_f_lipschitz} and $\|f\|_\mathcal{K}$, GP models satisfy \Cref{as:well_calibrated_model}. 
Furthermore, by Lemma 13 in Appendix G of \citet{curi2020efficient}, we now that such models also satisfy \Cref{as:model_predictions_lipschitz}. 
The remaining condition is that the results in previous sections assume that the domain is bounded. 
However, under \Cref{as:dynamics_f_lipschitz} this is not true. 
Fortunately, \citet{curi2020efficient} prove in Appendix I that the domain is bounded with high-probability. Furthermore, they prove that the MIG of kernels only increases poly-logarithmically, which does not affect the regret bounds in this paper.





\section{Extended Experimental Results} \label{sec:extra_experiments}

In this section, we detail the experimental procedures for completeness. 
We first detail how we learn the dynamical model using an ensemble of neural networks in \Cref{ssec:model_learning} as it is common to all experiments. We then detail the inverted pendulum experiment from \Cref{sec:introduction} in \Cref{ssec:inverted_pendulum}. 
We describe the adversarial-robust experiment in \Cref{ssec:adversarial}, the action-robust experiment in \Cref{ssec:action}, and the parameter-robust experiment in \Cref{ssec:parameter}. 
All experiments where run with 18-core Intel Xeon E5-2697v4 processors. 

\subsection{Model Learning and Calibration} \label{ssec:model_learning}
To learn the dynamics, we use a probabilistic ensemble with five heads as in \citet{chua2018deep}.
The model predicts the change in state, i.e, $\delta_{\hi} = \x_{\hi+1} - \x_{\hi}$ and we normalize the states, actions and change in next-states, with the running mean and standard deviation, similar in nature to \citet{van2016learning}. 
After each episode, we split the data into a train and validation set with a 0.9/0.1 ratio. 
For each ensemble member, we also sample a weight to simulate bootstrapping as in \citet{osband2016deep}. 
Finally, each model is trained minimizing the negative log-likelihood of a Gaussian distribution.
We train for 20 epochs and early stop if the prediction mean-squared-error when the epoch-loss on a validation set is $10 \%$ larger than the minimum epoch-loss in the same validation set. 
After training, we recalibrate on the validation set using temperature scaling. In particular, we use binary search to find the best parameter in the interval $[0.01, 100]$ that minimizes the expected calibration error \citep{malik2019calibrated}. 

\subsection{Inverted Pendulum Swing-Up Task} \label{ssec:inverted_pendulum}
The pendulum swing up task has a reward function given by $r(\x, \u) = -(\theta^2 + 0.1 * \dot{\theta}^2)$, where $\theta$ is the angle and $\dot{\theta}$ is the angular velocity. The Pendulum always starts from $\theta=\pi$ in the bottom down position and the goal is to swing the pendulum to the top-up position at $\theta=0$. Crucially, the initial distribution is a dirac-distribution located at $\theta=\pi$, i.e., it does not have enough coverage for algorithms to explore with it.

\paragraph{Adversarial-Robust.} In this setting, the adversary can change the relative gravity and the relative mass of the environment at every episode between $[1-\alpha, 1+\alpha]$, for varying $\alpha$. We train \rhucrl for 200 episodes, \hucrl with the nominal gravity and mass for 200 episodes, and the baseline in this setting is \texttt{RARL}, which we train for 1000 episodes. To evaluate the robust performance, we train \texttt{SAC} for 200 episodes, fixing the agent policy of the algorithms.

\paragraph{Action-Robust.} In this setting, the action is a mixture sampled with probability $\alpha$ of the learner and the adversary, i.e., the adversary only affects the input torque to the pendulum. The training and evaluation procedure is the same as in the adversarial robust setting. The baseline in this setting is \texttt{AR-DDPG}, which we train for 200 episodes.

\paragraph{Parameter-Robust.} In this setting, we consider robustness to mass change. Compared to the adversarial-robust setting, here the adversary is only allowed to change the mass once per episode. In this setting, \hucrl is trained for 200 episodes with the nominal mass, and then it is evaluated for varying masses.
\rhucrl and the baseline, \texttt{EP-Opt} are allowed to change the mass also during training. 
In this setting, there is no worst-case adversary during evaluation.

\subsection{Adversarial-Robust RL} \label{ssec:adversarial}

Next we detail the environments, the training and evaluation procedure, and the hyper-parameters in different paragraphs. 

\paragraph{Environments.}
For the Half-Cheetah environment, the adversary acts on the torso, the front foot and the back foot. For the Hopper environment, the adversary acts on the torso. For the Inverted Pendulum, the adversary acts on the pole. The Inverted Pendulum task is different here as it starts from a perturbation of the top-up position and the task is to stabilize the pendulum. 
For the Reacher2d environment, the adversary acts on the body0 link. For the Swimmer, the adversary acts on the torso. For the Walker, the adversary acts on the torso. For all environments, we use the adversarial input magnitude $\Uadv = [-10, 10]^{\ninpadv}$, where  $\ninpadv$ is environment dependent. 

\paragraph{Training and Evaluation.}
We train \rhucrl, \texttt{BestResponse}, \texttt{MaxiMin-MB}, \texttt{MaxiMin-MF} with in an adversarial environment for 200 episodes. We train \texttt{RARL} and \texttt{RAP} for 1000 episodes in an adversarial environment. Finally, we train \hucrl for 200 episodes in a standard environment.To evaluate the robust performance of each algorithm, we freeze the output policy of the training step and train an adversary using \texttt{SAC} for 200 episodes. 
We perform five independent runs and report the mean and standard deviation over the runs. 

\paragraph{Algorithm Hyper-Parameters.}
For \rhucrl and its variants, we fix $\beta=1.0$, we train every time step and do two gradient steps with Adam \citep{kingma2014adam} with learning rate $=3\times10^{-4}$. To compute a policy gradient, we take pathwise derivatives of a learned critic using the learned model for 3 time steps and weight each estimates using td-$\lambda$, with $\lambda=0.1$ \citep{sutton2018reinforcement}. 
We also add entropy regularization with parameter $0.2$. 
We did not do hyper parameter search, but rather use the software default values. 
For \texttt{RARL} and \texttt{RAP}, we use the \texttt{PPO} algorithm from \citet{schulman2017proximal} as this performed better than \texttt{TRPO} from \citet{schulman2015trust}.  
We train \texttt{PPO} after collecting a batch of 4 episodes, for 80 gradient steps, using early stopping once the KL divergence between the initial and the current policy is more than $0.0075$. 

\subsection{Action-Robust RL} \label{ssec:action}

We use the noisy robust setting from \citet{tessler2019action} with mixture parameter $\alpha=0.3$. 
The training and evaluation procedures, as well as the hyperparameters, are identical to the adversarial-robust experiment. 
The baseline is \texttt{AR-DDPG} that \citet{tessler2019action} propose.

\subsection{Parameter-Robust} \label{ssec:parameter}
In this setting, we are robust to a relative mass change, i.e., when the relative mass equals one, then the environment has the nominal mass.
For all environments except for the Swimmer, the relative mass is bounded in the interval $[0.001, 2]$. For the Swimmer, there were numerical errors due to ill-conditioned mass in the simulator, so we limit the range to $[0.5, 1.5]$. 

We train \rhucrl, \texttt{BestResponse}, \texttt{MaxiMin-MB}, \texttt{MaxiMin-MF} with in an adversary that is allowed to select the worst-case mass from within the range for 200 episodes. 
We train the baselines, \texttt{DomainRandomization} and \texttt{EPOpt} for 1000 episodes as they were also based on \texttt{PPO}. Finally, we train \hucrl for 200 episodes in a standard environment.
To evaluate the performance of each algorithm, we evaluate the different relative masses in the interval and do five independent runs. We report the mean and standard deviation over the runs.

\end{document}